\documentclass{article}

\usepackage[utf8]{inputenc}
\usepackage[margin=1in]{geometry}
\usepackage{amssymb, amsmath, amsthm, bm}
\usepackage{graphicx}
\usepackage{subcaption}
\usepackage{enumerate, xcolor}
\usepackage{algorithm, algorithmicx, algpseudocode}
\usepackage{xspace,color}
\usepackage{booktabs}
\usepackage{multirow}
\usepackage{hyperref}
\hypersetup{colorlinks,linkcolor=red,anchorcolor=blue,citecolor=blue}
\usepackage{setspace}
\usepackage{cleveref}
\usepackage{pifont}
\makeatletter

\theoremstyle{plain}

\newtheorem{theorem}{Theorem}%
\newtheorem{lemma}{Lemma}

\newtheorem{assumption}{Assumption}
\newtheorem{example}{Example}
\newtheorem{definition}{Definition}
\newtheorem{remark}{Remark}

\@ifpackageloaded{cleveref}{\crefname{theorem}{theorem}{theorems}}{}
\@ifpackageloaded{cleveref}{\crefname{lemma}{lemma}{lemmas}}{}
\@ifpackageloaded{cleveref}{\crefname{assumption}{assumption}{assumptions}}{}
\@ifpackageloaded{cleveref}{\crefname{corollary}{corollary}{corollaries}}{}
\@ifpackageloaded{cleveref}{\crefname{definition}{definition}{definitions}}{}

\newcommand*{\F}{{\mathsf{F}}}

\def\letters{a,b,c,d,e,f,g,h,i,j,k,l,m,n,o,p,q,r,s,t,u,v,w,x,y,z}
\def\Letters{A,B,C,D,E,F,G,H,I,J,K,L,M,N,O,P,Q,R,S,T,U,V,W,X,Y,Z}
\def\greekletters{alpha,beta,gamma,Gamma,lambda,Lambda,delta,Delta,zeta,kappa,phi,psi,nu,eta,varsigma,varphi,xi}
\edef\AllLetters{\letters,\Letters}

\def\DefineShortCuts#1#2#3#4#5{%
    \def\DefineShort##1;{\expandafter \providecommand \csname #2\string##1#3\endcsname {{#4{#5{##1}}}}}
    \@for\next:=#1\do{\expandafter\DefineShort\next;}}

\def\DefineShortCutsGreek#1#2#3#4#5{%
    \def\DefineShort##1;{\expandafter \providecommand \csname #2\string##1#3\endcsname {{#4{#5{\csname ##1\endcsname}}}}}
    \@for\next:=#1\do{\expandafter\DefineShort\next;}}

\DefineShortCuts{\AllLetters}{b}{}{\bm}{}
\DefineShortCuts{\AllLetters}{bb}{}{\overline}{\bm}
\DefineShortCuts{\Letters}{c}{}{\mathcal}{}
\DefineShortCuts{\Letters}{}{b}{\mathbb}{\bm}
\DefineShortCutsGreek{\greekletters}{b}{}{\bm}{}
\DefineShortCutsGreek{\greekletters}{bb}{}{\overline}{\bm}

\newcommand{\one}{{\bm 1}}
\newcommand\x\bx
\newcommand\y\by
\newcommand\E\Eb
\newcommand\R\Rb
\newcommand\I\Ib
\renewcommand\P\Pb %

\makeatother
 
\newcommand{\cmark}{\ding{51}}%
\newcommand{\xmark}{\ding{55}}%

\newcommand{\ave}{( \tfrac1n \bm 1_n)}
\newcommand{\mean}{( \tfrac1n \bm 1_n \bm 1_n^\top)}

\newcommand{\opt}{\star}
\newcommand{\clip}{\mathsf{Clip}}
\newcommand{\beer}{\texttt{BEER}\xspace}
\newcommand{\myalg}{\texttt{PORTER}\xspace}
\newcommand{\myalgdp}{\texttt{PORTER-DP}\xspace}
\newcommand{\myalgsgd}{\texttt{PORTER-GC}\xspace}

\newcommand{\soteriasgd}{\texttt{SoteriaFL-SGD}\xspace}

\newcommand{\dataset}[1]{#1\xspace}

\newcommand{\citet}[1]{\cite{#1}}
\newcommand{\citep}[1]{\cite{#1}}

\allowdisplaybreaks

\title{Convergence and Privacy of Decentralized Nonconvex Optimization with Gradient Clipping and Communication Compression}

\author{Boyue Li\thanks{Department of Electrical and Computer Engineering, Carnegie Mellon University, Pittsburgh, PA 15213, USA. Emails: \texttt{\{boyuel, yuejiec\}@andrew.cmu.edu.}}\\
 	Carnegie Mellon University\\
	\and
 	Yuejie Chi\footnotemark[1] \\ 	 
  	Carnegie Mellon University}

\begin{document}

\maketitle

\begin{abstract}

Achieving communication efficiency in decentralized machine learning has been attracting significant attention, with communication compression recognized as an effective technique in algorithm design. This paper takes a first step to understand the role of gradient clipping, a popular strategy in practice, in decentralized nonconvex optimization with communication compression. We propose \myalg, which considers two variants of gradient clipping added before or after taking a mini-batch of stochastic gradients, where the former variant \myalgdp allows local differential privacy analysis with additional Gaussian perturbation, and the latter variant \myalgsgd helps to stabilize training. We develop a novel analysis framework that establishes their convergence guarantees without assuming the stringent bounded gradient assumption. To the best of our knowledge, our work provides the first convergence analysis for decentralized nonconvex optimization with gradient clipping and communication compression, highlighting the trade-offs between convergence rate, compression ratio, network connectivity, and privacy.

\end{abstract}

\section{Introduction}

Decentralized machine learning has been attracting significant attention in recent years, 
which can be often modeled as a nonconvex finite-sum optimization problem:
\begin{align}
    \label{eq:problem_definition_finite_sum}
    \min_{\x \in \R^d} f(\x) := \frac1n \sum_{i=1}^n f_i(\x),
    \text{~~~~~where~~} f_i(\x) = \frac1m \sum_{\bz \in \cZ_i} \ell(\bx; \bz),
\end{align}
where $\x \in \R^d$ and $\bz$ denote the optimization parameter and one data sample,
$\ell(\x; \bz)$ denotes the sample loss function that is nonconvex in $\x$, and
$f_i(\x)$ and $f(\x)$ denote the local objective function at agent $i$
and the global objective function. In addition,
$\cZ_i$ denotes the dataset at agent $i$,
 $m = | \cZ_i |$ denotes the local sample size, and $n$ denotes the number of agents.
 An undirected communication graph $\cG$ is used to model the connectivity between any two agents,
where there is an edge between agent $i$ and $j$ only if they can communicate.
The goal is to efficiently optimize the global objective function $f(\x)$ in a decentralized manner,
subject to the network connectivity constraints specified by $\cG$.

Communication efficiency is critical to decentralized optimization algorithms,
as communication can quickly become bottleneck of the system as the number of agents and the size of the model increase. This has led to the development of communication compression (or quantization) techniques,
which can significantly reduce the communication burden per round by transferring compressed information, especially when the communication bandwidth is limited. Therefore, a number of recent works have focused on designing decentralized nonconvex optimization algorithms with communication compression, including but not limited to \cite{koloskova2019decentralizeddeep,singh2021squarmsgd,zhao2022beer,tang2019deepsqueeze,huang2023cedas,yan2023compressed}.

Built upon this line of work, the paper aims to understand the role of gradient clipping in decentralized nonconvex optimization algorithms with communication compression. On the one hand, gradient clipping has been used widely in privacy-preserving algorithms \cite{abadi2016deep} to enable (local) differential privacy guarantees \citep{dwork2008differential}. On the other hand, gradient clipping is also observed to be beneficial in stabilizing neural network training \cite{zhang2020gradient}. However, since gradient clipping necessarily introduces bias, the characterization of the convergence becomes much more challenging compared to their unclipped counterpart. As a result, most of the existing theoretical analyses for stochastic gradient algorithms with clipping --- in the context of centralized and server/client settings --- make strong assumptions such as the bounded gradient assumption \cite{abadi2016deep,wang2019efficient,li2022soteriafl} and the uniformly bounded gradient assumption \cite{yang2022normalized,zhang2020improved,zhang2020gradient}. To the best of our knowledge, the convergence of stochastic gradient algorithms with clipping in the decentralized setting has not been investigated before.

\subsection{Our contributions}
This paper proposes \myalg (cf. \Cref{alg:dp_dsgd}),\footnote{The name is coined for two reasons: 1) \myalg has strong connection to the prior algorithm \beer (porter is a kind of dark beer), and 2) the authors developed this algorithm in Porter Hall at Carnegie Mellon University.}
a communication-efficient decentralized algorithm for nonconvex finite-sum optimization with gradient clipping and communication compression. \myalg is built on \beer \cite{zhao2022beer} --- a fast decentralized algorithm with communication compression proposed recently --- by introducing gradient clipping to the local stochastic gradient computation at agents, while inheriting the desirable designs such as error feedback and stochastic gradient tracking that are crucial in enabling the fast convergence of \beer. \myalg considers two variants of gradient clipping, corresponding to adding it before or after taking a mini-batch of stochastic gradients. In particular, the former variant \myalgdp allows local differential privacy (LDP) analysis with additional Gaussian perturbation, and the latter variant \myalgsgd helps to stabilize training. Assuming a smooth clipping operator (\Cref{def:clipping_operator}) and general compression operators (\Cref{definition:general_compression_operator}), the highlights of our contributions are as follows. 
\begin{enumerate}
    \item We establish that \myalgdp (cf. \Cref{alg:dp_dsgd}) achieves $(\epsilon, \delta)$-LDP under appropriate Gaussian perturbation. Under the bounded gradient assumption (when gradient clipping can be ignored), \myalgdp converges in 
    average squared $\ell_2$ gradient norm as
$\frac1T \sum_{t \in [T]} \E \big\| \nabla f(\bbx^{(t)}) \big\|^2_2 \lesssim \rho^{-\frac43} (1 - \alpha)^{-\frac83} \phi_m^{-1}$
in $T = \phi_m^{-2}$ iterations, where $\bbx^{(t)}$ is the average parameter, $\phi_m = \frac{\sqrt{d \log(1/\delta)}}{m \epsilon}$ is the baseline utility for a centralized stochastic algorithm to achieve $(\epsilon, \delta)$-DP with $m$ data samples \cite{abadi2016deep}, $\rho \in (0, 1]$ is the compression ratio, and $\alpha \in [0,1)$ is the mixing rate of the topology. 

\item However, the bounded gradient assumption might be too stringent to hold in practice. Instead we further establish that under the local variance and bounded dissimilarity assumptions, \myalgdp converges in minimum $\ell_2$ gradient norm as
$\min_{t \in [T]} \E \big\| \nabla f(\bbx^{(t)}) \big\|_2 \lesssim \rho^{-\frac23} (1 - \alpha)^{-\frac43} \phi_m^{-1/2}$
 in $T = \phi_m^{-2}$ iterations.

 \item We establish that under the local variance and bounded dissimilarity assumptions, by properly choosing the mini-batch size, \myalgsgd converges in minimum $\ell_2$ gradient norm as $\min_{t \in [T]} \E \big\| \nabla f(\bbx^{(t)}) \big\|_2 \lesssim \rho^{-\frac23} (1 - \alpha)^{-\frac43} T^{-1/2},$
        which matches the convergence rate of classical centralized stochastic algorithms.

\end{enumerate}

Our work develops a novel analysis framework that establishes their convergence guarantees without assuming the stringent bounded gradient assumption,  highlighting comprehensive trade-offs between convergence rate, compression ratio, network connectivity, and privacy. To the best of our knowledge, our work provides the first private decentralized optimization algorithm with communication compression, and a systematic investigation of gradient clipping in the fully decentralized setting. \Cref{table:utility} provides a detailed comparison of \myalgdp with prior art on private server-client algorithms, where the bounded gradient assumption is all in effect except ours.

\begin{table}[htbp]
    \caption{
    Comparison of final utility upper bounds and communication complexities of different stochastic algorithms that achieve $(\epsilon, \delta)$-DP/LDP.
    The Big-$O$ notation (defined in \Cref{sub:paper_organization_and_notations}) is omitted for simplicity.
    DP-SGD is a single-server optimization algorithm that serves as a baseline,
    to show the overhead brought in by the distributed setting.
    DDP-SRM and Soteria-SGD are server/client distributed algorithms,
    but DDP-SRM doesn't use communication compression.
    \\
    {\footnotesize $^{(1)}$ $\theta = (1-\omega)^{-3/2} n^{-1/2}$, where $\omega$ is the  parameter for unbiased compression.}
}
    \label{table:utility}
    \centering

\begin{tabular}{c||c|c|c|c|c}

    \toprule

    \multirow{2}{*}{Algorithm}
    & \multirow{2}{*}{Privacy}
    & Compression
    & Bounded
    & \multirow{2}{*}{Utility}
    & Communication
    \\

    & 
    & operator
    & gradient
    & 
    & rounds
    \\

    \hline
    \hline

    DP-SGD
    & \multirow{2}{*}{DP}
    & \multirow{2}{*}{-}
    & \multirow{2}{*}{\cmark}
    & \multirow{2}{*}{$\phi_m$}
    & \multirow{2}{*}{-}
    \\

    \citep{abadi2016deep}
    & 
    & 
    & 
    & 
    & 
    \\

    \hline

    DDP-SRM
    & \multirow{2}{*}{DP}
    & \multirow{2}{*}{-}
    & \multirow{2}{*}{\cmark}
    & \multirow{2}{*}{$\frac1n \phi_m$}
    & \multirow{2}{*}{$n^2 d \phi_m^{-1}$}
    \\

    \citep{wang2019efficient}
    & 
    & 
    & 
    & 
    & 
    \\

    \hline

    Soteria-SGD $^{(1)}$
    & \multirow{2}{*}{LDP}
    & \multirow{2}{*}{Unbiased}
    & \multirow{2}{*}{\cmark}
    & \multirow{2}{*}{$(1 + \theta^{1/2}) \big(\frac{1 + \omega}{n}\big)^{1/2} \phi_m$}
    & \multirow{2}{*}{$(1 + \theta^{1/2}) \big(\frac{n}{1 + \omega} \big)^{2/3}  d \phi_m^{-1}$}
    \\

    \citep{li2022soteriafl}
    & 
    & 
    & 
    & 
    & 
    \\

    \hline

    \myalgdp
    & \multirow{2}{*}{LDP}
    & \multirow{2}{*}{General}
    & \multirow{2}{*}{\cmark}
    & \multirow{2}{*}{$\frac{1}{(1 - \alpha)^{8/3} \rho^{4/3}} \phi_m $}
    & \multirow{2}{*}{$\phi_m^{-2}$}
    \\

    (\Cref{alg:dp_dsgd})
    & 
    & 
    & 
    & 
    & 
    \\

    \hline

    \myalgdp
    & \multirow{2}{*}{LDP}
    & \multirow{2}{*}{General}
    & \multirow{2}{*}{\xmark}
    & \multirow{2}{*}{$\frac{1}{(1 - \alpha)^{16/3} \rho^{8/3}} \phi_m$}
    & \multirow{2}{*}{$\phi_m^{-2}$}
    \\

    (\Cref{alg:dp_dsgd})
    & 
    & 
    & 
    & 
    & 
    \\

    \bottomrule

    \hline

\end{tabular}
 \end{table}

\subsection{Related works}
Decentralized optimization algorithms have been extensively studied to solve large-scale optimization problems.
We review most closely related works in this section, and refer readers to more comprehensive reviews in \citet{nokleby2020scaling,xin2020general}.

\paragraph{Decentralized stochastic nonconvex optimization.}
Decentralized stochastic algorithms have been a actively researched area in recent years.
Various algorithms have been proposed by directly adapting existing centralized algorithms,
e.g., \citep{kempe2003gossipbased,xiao2004fast,shah2007gossip,bianchi2013convergence,lian2017can,assran2019stochastic,wang2021cooperative}.
However, the simple adaptations usually fail to achieve better convergence rates.
Gradient tracking \citep{zhu2010discrete},
originally proposed by the control theory community,
can be used to track the global gradient at each agent,
which leads to a simple systematic framework for extending existing centralized algorithms to the decentralized setting.
Gradient tracking can be used for both deterministic optimization algorithms, e.g., \citet{di2016next,nedic2017achieving,qu2018harnessing,li2020communication},
and stochastic algorithms, e.g., \citet{sun2020improving,xin2022fast,xin2022fasta,li2022destress,huang2022tackling,luo2022optimal}.

\paragraph{Communication compression.} In
\citet{de2015taming,alistarh2017qsgd}, gradient compression was adopted to create a server/client distributed SGD algorithm, however,
the large variance of compressed gradients leads to a sub-optimal convergence rate.
\citet{seide20141bit} first proposed the use of error feedback to compensate for the variance induced by compression.
\citet{stich2018sparsified,alistarh2018convergence,mishchenko2019distributeda,li2020acceleration,gorbunov2021marina,li2021canita} all equipped similar mechanism to improve convergence for server/client distributed optimization algorithms,
and \citet{richtarik2021ef21,fatkhullin2021ef21} formalized the error feedback mechanism and reaches an $O(1/T)$ convergence rate for smooth nonconvex objective functions.
\citet{tang2018communication,koloskova2019decentralizeddeep,singh2021squarmsgd,taheri2020quantized,zhao2022beer,yan2023compressed,liao2023linearly,zhao2021faster}
further extended communication compression schemes to the decentralized setting.

\paragraph{Private optimization algorithms.}
The concern of leaking sensitive data has been increasing with the rapid development of large-scale machine learning systems.
To address this concern,
the concept of differential privacy is widely adopted \citep{dwork2006calibrating,dwork2008differential}, where 
a popular approach to protect privacy is adding noise to the model or gradients. This approach is first adopted in the single server setting  to design differentially private optimization algorithms
\citep{abadi2016deep,wang2017differentially,iyengar2019practical,feldman2020private,chen2020understanding,wang2020differentially},
while \citet{huang2020dpfl,asoodeh2021differentially,noble2022differentiallya,du2022dynamic,li2022soteriafl,murata2023diff2,zhang2022understanding} considered differential privacy for the server/client distributed setting. 

\paragraph{Gradient clipping.}
Understanding gradient clipping has gained significant attention in recent years.
Earlier works, 
e.g. \citet{pascanu2013difficulty,bengio2013advances,kim2016accurate,kanai2017preventing,you2017scaling},
used gradient clipping as a pure heuristic to solve gradient vanishing/exploding problems without theoretical understandings.
Then,
\cite{zhang2020why,zhang2020whya,zhang2020improved,reisizadeh2023variancereduced} introduced theoretical analyses to understand its impact on the convergence rate and training performance.
This question is also investigated in  \cite{chen2020understanding,zhang2022understanding,das2022uniform,fang2023improved},
which applies gradient clipping to limit the magnitude of each of the sample gradients,
so that the variance of privacy perturbation can be decided without the bounded gradient assumption. While finishing up this paper, we became aware of  \cite{koloskova23revisiting}, which also develops convergence guarantees on the minimum $\ell_2$ gradient norm of clipped stochastic gradient algorithms in the centralized setting with a piece-wise linear clipping operator. In contrast, our focus is on the decentralized setting with a smooth clipping operator, where extra care is taken to deal with the discrepancy between the local and global objective functions.

\subsection{Paper organization and notation}
\label{sub:paper_organization_and_notations}

\Cref{sec:preliminaries} introduces preliminary concepts,
\Cref{sec:algorithm} describes the algorithm development,
\Cref{sec:theoretical_guarantees} shows the theoretical performance guarantees for \myalg,
\Cref{sec:numerical_experiments} provides numerical evidence to support the analysis,
and \Cref{sec:conclusions} concludes the paper.
Proofs are postponed to appendices.

Throughout this paper, we use uppercase and lowercase boldface letters to represent matrices and vectors, respectively.
We use $\| \cdot \|_{\mathsf{op}}$ for matrix operator norm,
$\| \cdot \|_\F$ for Frobenius norm,
$\bI_n$ for the $n$-dimensional identity matrix,
$\bm1_n$ for the $n$-dimensional all-one vector
and $\bm0_{d \times n}$ for the $(d \times n)$-dimensional zero matrix.
For two real functions $f(\cdot)$ and $g(\cdot)$ defined on $\R^+$,
we say $f(x) = O \big( g(x) \big)$ or $f(x) \lesssim g(x)$ if there exists some universal constant $M > 0$ such that
$f(x)  \leq M g(x)$. The notation $f(x) =\Omega \big( g(x) \big)$ or $f(x)\gtrsim g(x)$ means $g(x) = O\big(f(x) \big)$.

\section{Preliminaries}
\label{sec:preliminaries}

\paragraph{Mixing.}%
The information mixing between agents is conducted by updating the local information via a weighted sum of information from neighbors,
which is characterized by a mixing (gossiping) matrix. Concerning this matrix is an important quantity called the mixing rate, defined in \Cref{definition:mixing_matrix}.
\begin{definition}[Mixing matrix and mixing rate]
    \label{definition:mixing_matrix}
    The {\em mixing matrix} is a matrix $\bW = [w_{ij}] \in \R^{n \times n}$,
    such that $w_{ij} = 0$ if agent $i$ and $j$ are not connected according to the communication graph $\cG$. Furthermore,
    $\bW \bm1_n = \bm1_n$ and $\bW^\top \bm1_n = \bm1_n$.
The {\em mixing rate} of a mixing matrix $\bW$ is defined as
    \begin{align}
        \alpha := \big\|\bW  - \tfrac1n \bm1_n\bm1_n^\top \big\|_{\mathsf{op}}.
        \label{eq:def_alpha0}
    \end{align}
\end{definition}

The mixing rate describes the connectivity of a communication graph and the speed of information sharing.
Generally,
a better connected graph leads to a smaller mixing rate,
for example,
$\bW$ can be the averaging matrix for a fully connected communication network,
which results in $\alpha = 0$.
A comprehensive list of bounds on $1-\alpha$ is provided by \cite[Proposition 5]{nedic2018network}.
Our analysis does not require the mixing matrix to be doubly stochastic,
while allows us to use a non-symmetric matrix with negative values as the mixing matrix, such as the FDLA matrix \citep{xiao2004fast}, which has a smaller mixing rate under the same connectivity pattern.

\paragraph{Gradient clipping.}
In practice, gradient clipping is frequently adopted to ensure the gradients are within a predetermined region,
so that the variance of privacy perturbation can be decided accordingly.
The clipping operator we adopt is a smooth clipping operator \citep{yang2022normalized} defined in \Cref{def:clipping_operator},
which scales a vector into a ball of radius $\tau$ centered at the origin.

\begin{definition}[Smooth clipping operator]
    \label{def:clipping_operator}
    For $\x \in \R^d$,
    the clipping operator is defined as
    \begin{align*}
        \clip_\tau(\x) = \frac{\tau}{\tau + \| \x \|_2} \x .
    \end{align*}
    For $\bX=[\x_1,\ldots, \x_n] \in \R^{d \times n}$,
    the distributed clipping operator is defined as
    \begin{align*}
        \clip_\tau(\bX) = [\clip_\tau(\x_1), \ldots, \clip_\tau(\x_n)].
    \end{align*}
\end{definition}

\begin{remark}
Another widely used clipping operator is the piece-wise linear clipping operator,
which scales inputs whenever its gradient norm is larger than $\tau$ and does nothing otherwise, defined by
    \begin{align*}
        \clip_\tau(\x) = \x \min \big\{1, \tau / \| \x \|_2 \big\} .
    \end{align*}
 \Cref{fig:clipping} plots the norm of a vector before and after clipping for these two clipping operators, which show that they behave quite similarly. %
\end{remark}

\begin{figure}[t]
    \centering
    \includegraphics[width=0.45\linewidth]{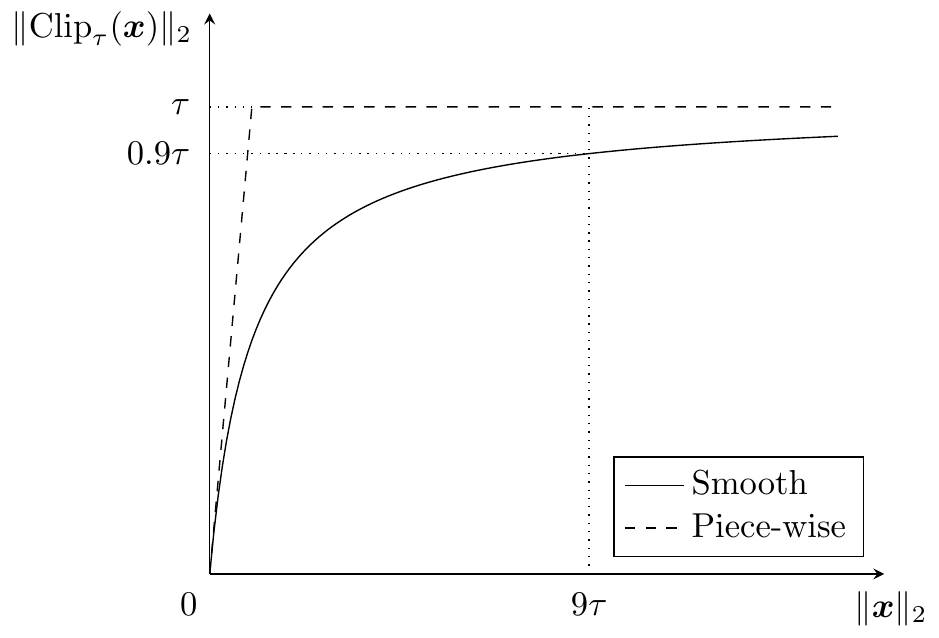}
    \caption[Illustration of input norm and clipped norm for clipping operators]{Illustration of input norm and clipped norm for the smooth clipping operator (\Cref{def:clipping_operator}) and piece-wise linear clipping operator,
    where $\tau$ is the clipping parameter.}
    \label{fig:clipping}
\end{figure}

\paragraph{Compression operators.}
Following \cite{richtarik2021ef21,fatkhullin2021ef21},
\Cref{definition:general_compression_operator} defines a randomized general compression operator that only guarantees the expected compression error $\E \| \cC(\x) - \x \|_2^2$ is less than the magnitude of original message $\| \x \|_2^2$.
\begin{definition}[General compression operator]
    \label{definition:general_compression_operator}
    A randomized map $\cC$ : $\R^d \to \R^d$ is a $\rho$-compression operator if $\forall \x \in \R^d$ and some $\rho \in [0, 1]$,
    the following inequality holds:
    \begin{align*}
    \E \big\| \cC(\x) - \x \big\|^2_2 \leq (1 - \rho) \| \x \|^2_2 .
    \end{align*}
\end{definition}

Many widely used compression schemes can be modeled as special cases, for example, random sparsificiation and top-$k$ compression.

\begin{example}[Random sparsification]
    \label{eg:random_sparsification}
    Random sparsification keeps an element from a $d$-dimensional vector with probability $\frac kd$. 
    Let $\bu \in \R^d$ where $u_i \sim B\big(\frac{k}{d}\big)$, then random sparsification is defined as
            $\mathrm{random}_{k}(\x) = \bu \odot \x$,
            which satisfies \Cref{definition:general_compression_operator} with $\rho = \frac kd$.

\end{example}

\newcommand*{\gsgd}{\mathrm{gsgd}}
\newcommand*{\topk}{\mathrm{top}_k}

\begin{example}[top$_k$]
    \label{example:topk}
    $\topk$ \citep{alistarh2018convergence,stich2018sparsified}
    keeps $k$ elements that have the largest absolute values and sets other elements to $0$,
    which is defined as
$\topk(\x) := \x \odot \bu(\x)$,
    where $[ \bu(\x) ]_i = 1$ if the absolute value of the $i$-th element is one of the $k$-largest absolute values,
    otherwise $[ \bu(\x) ]_i = 0$. It follows that
    $\topk$ satisfies Definition \ref{definition:general_compression_operator} with $\rho = k/d$.
\end{example}

\paragraph{Local differential privacy.} 
In decentralized learning systems,
all agents share information with their neighbors that are potentially sensitive.
If some agents are exploited by adversaries,
the system will face a risk of privacy leakage even when the system-level privacy is protected.
Therefore, we introduce local differential privacy (LDP) --- defined in \Cref{def:local_differential_privacy} --- following
\citet{duchi2013local,chatzikokolakis2013broadening,xiao2019local},
which protects each agent's privacy from leaking to other agents.

\begin{definition}[Local differential privacy (LDP)]
    \label{def:local_differential_privacy}
    A randomized mechanism $\cM: \cZ \to \cR$ with domain $\cZ$ and range $\cR$ satisfies $(\epsilon, \delta)$-local differential privacy for client $i$,
    if for any two neighboring dataset $\bZ_i, \bZ_i^{\prime} \in \cZ$ at client $i$ and for any subset of outputs $\bR \subseteq \cR$,
    it holds that
    \begin{align}
        \P\big( \cM(\bZ_i) \in \bR \big) \leq e^{\epsilon} \P\big( \cM(\bZ_i^{\prime}) \in \bR \big) + \delta.
    \end{align}
   The two datasets $\bZ_i $ and $\bZ_{i}^{\prime}$ are called neighboring if
they are only different by one data point at client $i$.
\end{definition}
\Cref{def:local_differential_privacy} is a stricter privacy guarantee because it can imply general differential privacy (DP).
Consequently,
LDP requires a larger perturbation variance than general DP.
To identify the impact of the decentralized LDP setting compared to centralized DP setting,
we define the baseline utility 
\begin{equation}\label{def:phi_m}
\phi_m = \frac{\sqrt{d \log(1/\delta)}}{m \epsilon},
\end{equation} 
which can be understood as the final utility of a centralized system with $m$ data samples that guarantees $(\epsilon, \delta)$-DP.
For typical private problems,
the local sample size $m$ has to be large enough for the privacy perturbation to deliver meaning guarantees,
we impose a mild assumption that $\phi_m < 1$ to simplify presentation.
For example,
the problem defined in \eqref{eq:problem_definition_finite_sum} has in total $mn$ data samples,
running an $(\epsilon, \delta)$-DP algorithm on one server that can access all data will achieve $\frac1n \phi_m$ utility in $n \phi_m^{-1}$ iterations.

\section{Proposed algorithm}
\label{sec:algorithm}

We propose \myalg, a novel stochastic decentralized optimization algorithm for finding first-order stationary points of nonconvex finite-sum problems with gradient clipping and communication compression; the details are described in \Cref{alg:dp_dsgd}. On a high level, \myalg is composed of local stochastic gradient updates and neighboring information sharing, following a similar structure as \beer \citep{zhao2022beer}, in terms of the use of error feedback \citep{richtarik2021ef21}, which accelerates the convergence with biased compression operators, and stochastic gradient tracking to track the global gradient locally at each agent. A key difference, however, is the use of gradient clipping. Motivated by efficient training and privacy preserving, we consider two variants, corresponding to clipping {\em before} the mini-batch with privacy perturbation (\myalgdp), and {\em after} taking the mini-batch (\myalgsgd), respectively.

\begin{algorithm}[htbp]
\setstretch{1.25}
\caption{\myalg}
\label{alg:dp_dsgd}
\begin{algorithmic}[1]
    \State {\textbf{input:} $\bbx^{(0)}$, gradient stepsize $\eta$, consensus stepsize $\gamma$, clipping threshold $\tau$, mini-batch size $b$, perturbation noise $\sigma_p$, number of iterations $T$}

    \State {\textbf{initialize:} $\bV^{(0)} = \bQ_v^{(0)} = \bG_p^{(0)} = \bm0_{d \times n}$, $\bQ_x^{(0)} = \bX^{(0)} = \bbx^{(0)} \one_n^\top$}

    \For {$t = 1, \ldots, T$}
	\State Draw the local mini-batch of size $b$ uniformly at random $\cZ^{(t)}= \{\cZ_i^{(t)}\}_{i=1}^n$
        \State \textbf{Option I: \myalgdp (differentially-private SGD)}  
        \State \qquad$\bG_\tau^{(t)} = \frac{1}{b} \sum_{\bZ \in \cZ^{(t)}}  \clip_\tau (\nabla \ell (\bX^{(t-1)}; \bZ ))$
         \label{alg:eq:gradient_3.5}
        \State \qquad $\bG_p^{(t)} = \bG_\tau^{(t)} + \bE^{(t)}$, where $\be_i^{(t)} \sim \cN(\bm0_d, \sigma_p^2 \bI_d)$
        \label{alg:eq:gradient_3}

        \State \textbf{Option II: \myalgsgd (SGD with gradient clipping) }  
        \State \qquad $\bG^{(t)} = \frac{1}{b} \sum_{\bZ \in \cZ^{(t)}} \nabla \ell (\bX^{(t-1)}; \bZ )$
        \label{alg:eq:gradient_1}

        \State \qquad $\bG_p^{(t)} =\bG_\tau^{(t)} = \clip_\tau(\bG^{(t)})$
        \label{alg:eq:gradient_2}

        \State {$\bQ_v^{(t)} = \bQ_v^{(t-1)} + \cC (\bV^{(t-1)} - \bQ_v^{(t-1)})$}
        \label{alg:eq:gradient_estimate_1}
        \Comment{Communication}
        \State $\bV^{(t)} = \bV^{(t-1)} + \gamma \bQ_v^{(t)} (\bW - \bI_n) + \bG_p^{(t)} - \bG_p^{(t-1)}$
        \label{alg:eq:gradient_estimate_2}

        \State $\bQ_x^{(t)} = \bQ_x^{(t-1)} + \cC (\bX^{(t-1)} - \bQ_x^{(t-1)})$
        \label{alg:eq:variable_estimate_1}
        \Comment{Communication}
        \State $\bX^{(t)} = \bX^{(t-1)} + \gamma \bQ_x^{(t)} (\bW - \bI_n) - \eta \bV^{(t)} $
        \label{alg:eq:variable_estimate_2}

    \EndFor
    \State {\textbf{output:} ~$\x_{out} \sim \text{Uniform} (\{\bx_i^{(t)} | i \in [n], t \in [T] \})$.}	

\end{algorithmic}
\end{algorithm}

Before proceeding, we introduce some notation convenient for describing decentralized algorithms. Let $\x_i \in \R^d$ be the optimization variable at agent $i$,
we define the collection of all optimization variables as a matrix $\bX = [\x_1, \x_2, \hdots, \x_n] \in \R^{d \times n}$,
and the average as $\bbx = \frac1n \sum_{i=1}^n \x_i$.
The gradient estimates $\bV$, stochastic gradients $\bG$,
perturbation noise $\bE$,
compressed surrogate $\bQ_x$ and $\bQ_v$,
and their corresponding agent-wise values are defined analogously.
The distributed gradient is defined as $\nabla F(\bX) = \big[ \nabla f_1(\x_1), \nabla f_2(\x_2), \hdots, \nabla f_n(\x_n) \big] \in \R^{d\times n}$.

To provide more detail, \myalg initializes gradient-related variables to $\bm0_d$ and other variables to the same random value $\bbx^{(0)}$,
which improves stability in early iterations and simplifies analysis,
but has no impact on the convergence rates.
Within each iteration,
\myalg is consisted of $3$ major steps.
\begin{enumerate}[(1)]
    \item \textbf{Computing clipped stochastic gradients.} We consider two options. The first option \myalgdp corresponds to differentially-private SGD (\Cref{alg:eq:gradient_3.5}-\Cref{alg:eq:gradient_3}), where \Cref{alg:eq:gradient_3.5} computes a batch of clipped stochastic gradient on each agent, and
        then \Cref{alg:eq:gradient_3} adds Gaussian noise to ensure privacy.
        The second option \myalgsgd corresponds to SGD with gradient clipping (\Cref{alg:eq:gradient_1}-\Cref{alg:eq:gradient_2}), where
        \Cref{alg:eq:gradient_1} computes a batch stochastic gradient on each agent, and
        \Cref{alg:eq:gradient_2} applies clipping to each batch stochastic gradient. 

    \item \textbf{Updating gradient estimates.}
        \Cref{alg:eq:gradient_estimate_1} updates the auxiliary variable $\bQ_v^{(t)}$,
        which is a compressed surrogate of $\bV^{(t-1)}$,
        by adding the compressed difference to itself $\cC (\bV^{(t-1)} - \bQ_v^{(t-1)})$.
        Meanwhile,
        each agent $i$ sends its compressed difference $\cC (\bv_i^{(t-1)} - \bq_{v, i}^{(t-1)})$ to all of its neighbors,
        so that every neighbor can reconstruct the auxiliary variable $\bq_{v, i}^{(t)}$ by accumulating this difference.
        \Cref{alg:eq:gradient_estimate_2} then adds a correction term $\gamma \bQ_v^{(t)} (\bW - \bI_n)$,
        and applies stochastic gradient tracking to update gradient estimates.

    \item \textbf{Updating variable estimates.}
        Similar to updating gradient estimates,
        \Cref{alg:eq:variable_estimate_1} updates the auxiliary variable $\bQ_x^{(t)}$,
        which is a compressed surrogate of variable estimates $\bX^{(t-1)}$,
        and communicates with neighbors.
        \Cref{alg:eq:variable_estimate_2} applies correction and updates the variable estimates by a gradient-style update.
\end{enumerate}

\section{Theoretical guarantees}
\label{sec:theoretical_guarantees}
This section theoretically analyzes the privacy and convergence properties of \myalg under various assumptions.
\Cref{sub:assumptions} lists all assumptions required for convergence analysis,
 \Cref{sub:convergence_guarantees_with_privacy} shows the privacy and convergence of \myalgdp using a specific perturbation variance,
and \Cref{sub:convergence_guarantees_without_privacy}
shows the convergence of \myalg corresponding to clipped SGD without privacy. %

\subsection{Assumptions}
\label{sub:assumptions}

We start with smoothness assumption in \Cref{assumption:smoothness}, which is standard and required for all of our analysis.
\begin{assumption}[$L$-smoothness]
    \label{assumption:smoothness}
    For any $\x, \y \in \R^d$ and any datum $\bz$ in dataset $\cZ$,
    $$\| \nabla \ell(\x; \bz) - \nabla \ell(\y; \bz) \|_2 \leq L \|\x - \y\|_2.$$
\end{assumption}

Note that the gradient clipping operator $\clip_\tau(\cdot)$ is utilized to ensure gradients are bounded. In addition, the boundedness of the gradient ensures the application of differentially-private mechanisms. However, stochastic gradients at different agents lose correct scaling after clipping, which breaks the stationary point property at local minima and introduces bias. To simplify analysis, one assumption that has been adopted widely in theoretical analysis \cite{li2022soteriafl} is the following bounded gradient assumption.

\begin{assumption}[Bounded gradient]
    \label{assumption:bounded_stochastic_gradient}
    For any $\x \in \R^d$ and any datum $\bz$ in dataset $\cZ$,
    $\| \nabla \ell(\x; \bz) \|_2 \leq \tau$.
\end{assumption}

Under \Cref{assumption:bounded_stochastic_gradient},
\myalg can skip the clipping operator,
and $\bg_{\tau_i}^{(t)}$ becomes an unbiased estimator of local gradient $\nabla f_i(\bx_i^{(t)})$, while still allowing privacy analysis.
However, \Cref{assumption:bounded_stochastic_gradient} is rather strong and seldomly met in practice.
For example, the gradient of a quadratic loss function is not bounded. In addition, it may result in an overly pessimistic clipping operation when there are possibly adversarial gradients with large norms in the samples.  
Going beyond the strong bounded gradient assumption, we consider a much milder assumption that bounds the local variance as follows, which is more standard in the analysis of unclipped stochastic algorithms.
\begin{assumption}[Bounded local variance]
    \label{assumption:bounded_variance}
    For any $ \x \in \R^d$ and
    $i \in [n]$,
    $$\E_{\bz \sim \cZ_i} \| \nabla \ell(\x; \bz) - \nabla f_i(\x) \|_2^2 \leq \sigma_g^2.$$
\end{assumption}

An additional challenge is associated with dealing with the decentralized environment, where the local objective functions can be rather distinct from the global one. Our analysis identifies the following assumption, called bounded gradient dissimilarity, which says that the difference between the local gradient and the global gradient should be small relative to the global one.
\begin{assumption}[Bounded gradient dissimilarity]
    \label{assumption:bounded_similarity}
    For any $ \x \in \R^d$ and
    $i \in [n]$,
    $$\big\| \nabla f(\x) - \nabla f_i(\x) \big\|_2 \leq \frac{1}{12} \big\| \nabla f(\x) \big\|_2. $$
\end{assumption}

\subsection{Privacy and convergence guarantees of \myalgdp}
\label{sub:convergence_guarantees_with_privacy}

We start by analyzing the privacy and convergence guarantees of \myalgdp, assuming the batch size $b=1$.

\paragraph{Privacy guarantee.}

\Cref{theorem:privacy} proves that \myalgdp is $(\epsilon, \delta)$-LDP when setting the variance of Gaussian perturbation properly.

\begin{theorem}[Local differential privacy]
    \label{theorem:privacy}
    Let $\phi_m = \frac{\sqrt{d \log(1/\delta)}}{m \epsilon}$ and $b = 1$.
    For any $\epsilon \leq T / m^2$ and $\delta \in (0, 1)$,
    \myalgdp is $(\epsilon, \delta)$-LDP after $T$ iterations if we set
    \begin{align}
        \sigma_p^2 = \frac{T \tau^2 \log(1/\delta)}{m^2 \epsilon^2} = T \tau^2 \phi_m^2 / d .
        \label{eq:sigma_p}
    \end{align}
\end{theorem}

\Cref{theorem:privacy} shows that \myalgdp can achieve $(\epsilon, \delta)$-LDP regardless of whether the bounded gradient assumption presents,
because using the clipping operator $\clip_\tau(\cdot)$ can guarantee all the stochastic gradients' $\ell_2$ norms are bounded by $\tau$, 
so that the perturbation variance can be set accordingly.

\paragraph{Convergence with bounded gradient assumption.}
\label{sub:convergence_with_bounded_gradient_assumption}
We start by analyzing the convergence when the gradients are bounded under
 \Cref{assumption:bounded_stochastic_gradient}, in which case
\myalgdp can omit the clipping operator. %
\Cref{theorem:convergence_witht_privacy_with_bounded_gradient} presents the convergence result of \myalgdp using general compression operators (\Cref{definition:general_compression_operator}).

\begin{theorem}[Convergence of \myalgdp with bounded gradient assumption]
    \label{theorem:convergence_witht_privacy_with_bounded_gradient}

    Assume \Cref{assumption:smoothness,assumption:bounded_stochastic_gradient} hold,
    and use general compression operators (\Cref{definition:general_compression_operator}).
    Let $\Delta = \E[f(\bbx^{(0)})] - f^\star$.
    Set $\gamma = O\big( (1-\alpha) \rho \big)$,
    $\eta = O \big( \gamma^{4/3} \rho^{4/3} \phi_m \big)$,
    $T = \phi_m^{-2}$,
    $b = 1$
    and $\sigma_p^2 = T \tau^2 \phi_m^2 / d$.
   \myalgdp converges in average squared $\ell_2$ gradient norm as
    \begin{align}
        \frac1T \sum_{t=1}^T \E \| \nabla f(\bbx^{(t)}) \|_2^2
        &\lesssim \frac{\phi_m}{\rho^{\frac43} (1 - \alpha)^{\frac83}} \cdot \max \big\{\tau^2, L\Delta \big\} .
        \label{eq:convergence_with_bounded_gradient}
    \end{align}
\end{theorem}

\Cref{theorem:convergence_witht_privacy_with_bounded_gradient} shows the convergence error of the squared $\ell_2$ gradient norm with explicit dependency on the  compression ratio $\rho$ and  mixing rate $\alpha$.
When we fix $\rho$ and $\alpha$,
\Cref{theorem:convergence_witht_privacy_with_bounded_gradient} reaches an $O(\phi_m)$ final average squared $\ell_2$ gradient norm,
which matches the result of \soteriasgd \cite{li2022soteriafl}, the state-of-the-art stochastic  algorithm  with local differential privacy guarantees and {\em unbiased} communication compression in the server-client setting.
However,
due to extra complexities induced by the decentralized setting and {\em biased} compression,
\myalgdp takes $O(\phi_m^{-2})$ iterations to converge while \soteriasgd only takes $O(\phi_m^{-1})$ iterations; in addition,
 \myalgdp has a slightly worse dependency on the compression ratio $\rho$.

\paragraph{Convergence with bounded gradient assumption.} A more interesting and challenging scenario is when \Cref{assumption:bounded_stochastic_gradient} does not hold,
\myalgdp applies gradient clipping to ensure gradients are bounded to suit the privacy constraints. Fortunately, \Cref{theorem:convergence_witht_privacy_without_bounded_gradient} describes the convergence behavior of \Cref{alg:dp_dsgd} in this case, under the much weaker bounded local variance and bounded dissimilarity assumptions.

\begin{theorem}[Convergence of \myalgdp without bounded gradient assumption]
    \label{theorem:convergence_witht_privacy_without_bounded_gradient}
    Assume \Cref{assumption:smoothness,assumption:bounded_variance,assumption:bounded_similarity} hold,
    and use general compression operators (\Cref{definition:general_compression_operator}).
    Let $\Delta = \E[f(\bbx^{(0)})] - f^\star$. Set
   $\tau = \max \big\{ 365 \rho^{-\frac43} (1 - \alpha)^{-\frac83} \phi_m^{1/2}, 24 \sigma_g \big\}$,
    $\gamma = O((1 - \alpha)\rho)$,
    $\eta = O \big( L^{-1} \big)$,
    $b = 1$,
    $T= \phi_m^{-2}$
    and $\sigma_p^2 = T \tau^2 \phi_m^2 / d$.
    \Cref{alg:dp_dsgd} converges in minimum $\ell_2$ gradient norm as
    \begin{align}
        \min_{t \in [T]} \E \big\| \nabla f(\bbx^{(t)}) \big\|_2 \lesssim \max \big\{ \rho^{-\frac43} (1 - \alpha)^{-\frac83} (L \Delta \phi_m)^{1/2},~ \sigma_g \big\} .
        \label{eq:convergence_without_bounded_gradient}
    \end{align}
\end{theorem}

\Cref{theorem:convergence_witht_privacy_without_bounded_gradient} shows \myalgdp converges in minimum $\ell_2$ gradient norm with explicit dependency on compression ratio $\rho$, mixing rate $\alpha$ and gradient variance $\sigma_g$,  under much weaker assumptions.
To compare with \Cref{theorem:convergence_witht_privacy_with_bounded_gradient},
we can take the square root of \eqref{eq:convergence_with_bounded_gradient},
which translates to minimum $\ell_2$ gradient norm convergence on the order of
$
    O\big( \rho^{-2/3} (1 - \alpha)^{-4/3} \phi_m^{1/2} \cdot \max \{\tau, (L \Delta)^{1/2} \} \big) .
$
In comparison, although \Cref{theorem:convergence_witht_privacy_without_bounded_gradient} has worse dependency on compression ratio $\rho$ and mixing rate $\alpha$, it
matches the dependency on the baseline privacy loss $\phi_m$.

\subsection{Convergence guarantees of \myalgsgd}
\label{sub:convergence_guarantees_without_privacy}

\Cref{theorem:convergence_without_privacy_without_bounded_gradient} further establishes the convergence of \myalgsgd without the bounded gradient assumption, 
which applies the clipping operator to mini-batch stochastic gradients without privacy perturbation.
\begin{theorem}[Convergence of \myalgsgd without bounded gradient assumption]
    \label{theorem:convergence_without_privacy_without_bounded_gradient}
    Assume \Cref{assumption:smoothness,assumption:bounded_variance,assumption:bounded_similarity} hold,
    and use general compression operators (\Cref{definition:general_compression_operator}).
    Let $\Delta = \E[f(\bbx^{(0)})] - f^\star$.
    Set
    $\tau = O \big( \rho^{-\frac23} (1 - \alpha)^{-\frac43} T^{-1/2} \big)$,
    $\gamma = O \big((1 - \alpha)\rho \big)$,
    $\eta = O ( L^{-1} )$
    and $b = O(\sigma_g^2 \nu^{-2})$.
   \myalgsgd converges in minimum $\ell_2$ gradient norm as
    \begin{align*}
    \min_{t \in [T]} \E \big\| \nabla f(\bbx^{(t)}) \big\|_2 \lesssim \frac{1}{\rho^{\frac23} (1 - \alpha)^{\frac43}} \cdot \frac{1}{T^{1/2}} .
    \end{align*}
\end{theorem}

\Cref{theorem:convergence_without_privacy_without_bounded_gradient} suggests that by picking the clipping threshold $\tau$ and batch size $b$ properly, \myalgsgd converges at an $O(1/T^{1/2})$ rate. In comparison, standard centralized SGD converges in average squared $\ell_2$ gradient norm at an $O(1/T)$ rate,
which also translates to a minimum $\ell_2$ gradient norm convergence in the form of $\min_{t \in [T]} \E \big\| \nabla f(\x^{(t)}) \big\|_2 \lesssim 1/T^{1/2}$.
Therefore, using gradient clipping for decentralized SGD does not affect the convergence rate, providing proper hyper parameter choices.

When the gradients are bounded, we can omit the clipping operator in \myalgsgd, which become the same as \beer \citep{zhao2022beer}. Recall that
\beer guarantees a minimum $\ell_2$ gradient norm convergence at the rate $ \frac{1}{\rho^{1/2} (1-\alpha)^{3/2}} \cdot \frac{1}{T^{1/2}}$. In comparison,
\Cref{theorem:convergence_without_privacy_without_bounded_gradient} 
has a better dependency on the mixing rate $\alpha$,
but has a slightly worse dependency on the compression ratio $\rho$,
which again emphasizes that gradient clipping does not harm convergence. %
\section{Numerical experiments}
\label{sec:numerical_experiments}

This section presents numerical experiments to examine the performance of \myalgdp,
with comparison to the state-of-the-art server/client private stochastic algorithm \soteriasgd,
which also utilizes communication compression and guarantees local differential privacy.
More specifically,
we evaluate the convergence of utility and accuracy in terms of communication rounds and communication bits,
to analyze the privacy-utility-communication trade-offs of different algorithms.

For all experiments, we split shuffled datasets evenly to $10$ agents that are connected by an Erd\H{o}s-R\'enyi random graph with connecting probability $p = 0.8$.
We use the FDLA matrix \citep{xiao2004fast} as the mixing matrix to perform weighted information aggregation to accelerate convergence. 
We use biased random sparsification (cf. \Cref{eg:random_sparsification}) for all algorithms where $k = \lfloor \frac{d}{20} \rfloor$,
i.e., the compressor randomly selects 5\% elements from each vector.
We also apply gradient clipping with $\tau = 1$ to all algorithms for simplicity.
For each experiment,
all algorithms are initialized to the same starting point,
and use best-tuned learning rates, batch size $1$ and $\sigma_p = \frac{\tau \sqrt{T \log(1/\delta)}}{m \epsilon}$.

\subsection{Logistic regression with nonconvex regularization}
\label{sub:privacy_logistic}

We run experiments on logistic regression with nonconvex regularization on the \dataset{a9a} dataset \citep{chang2011libsvm}.
Following \citet{wang2019spiderboost},
the objective function is defined as
\begin{align*}
    \ell(\x; \{\bm f,l \}) = \log\left(1+ l \exp(-\x^\top \bm f)\right) + \lambda \sum \limits_{i=1}^d\frac{x_i^2}{1+x_i^2},
\end{align*}
where $\{ \bm f, l \}$ represents a training tuple,
$\bm f \in \R^d$ is the feature vector and $l \in \{0, 1\}$ is the label, and $\lambda$ is the regularization parameter which is set to $\lambda = 0.2$.

\begin{figure}[ht]
\begin{tabular}{cc}
\includegraphics[width=0.48\textwidth]{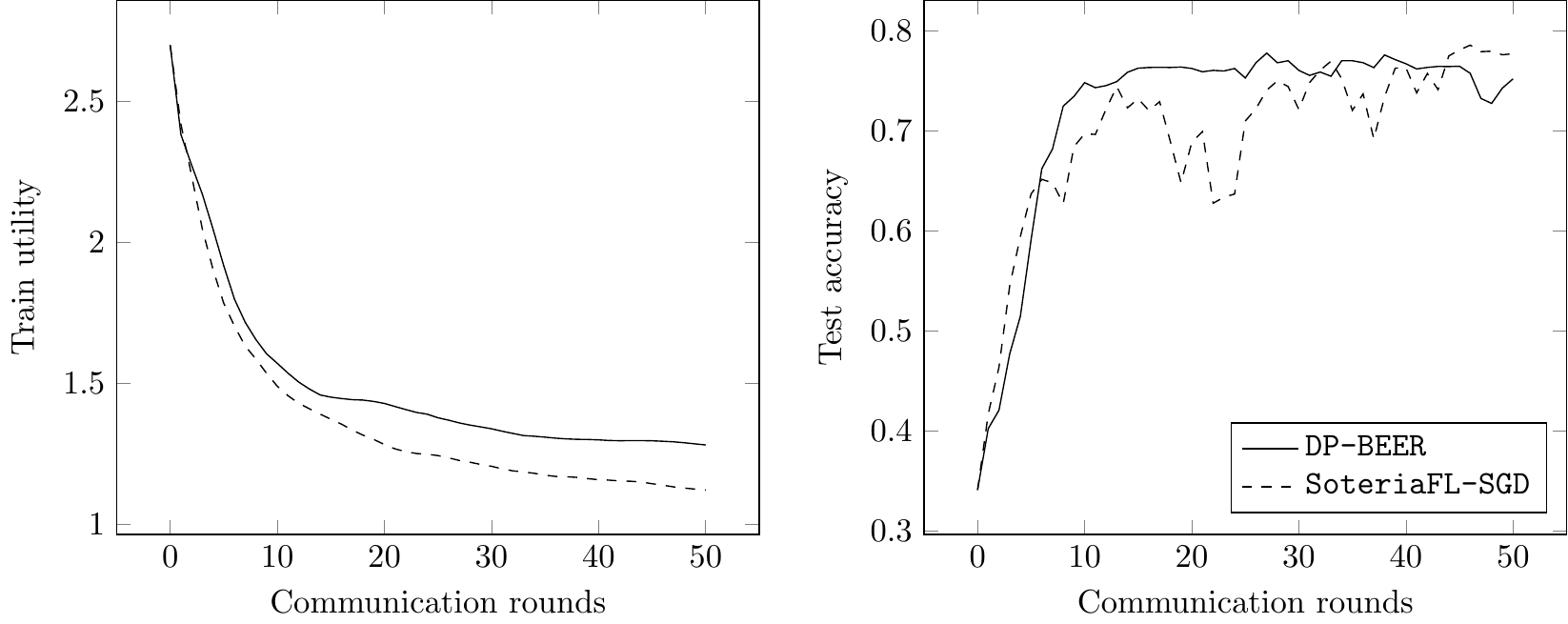} &\includegraphics[width=0.48\textwidth]{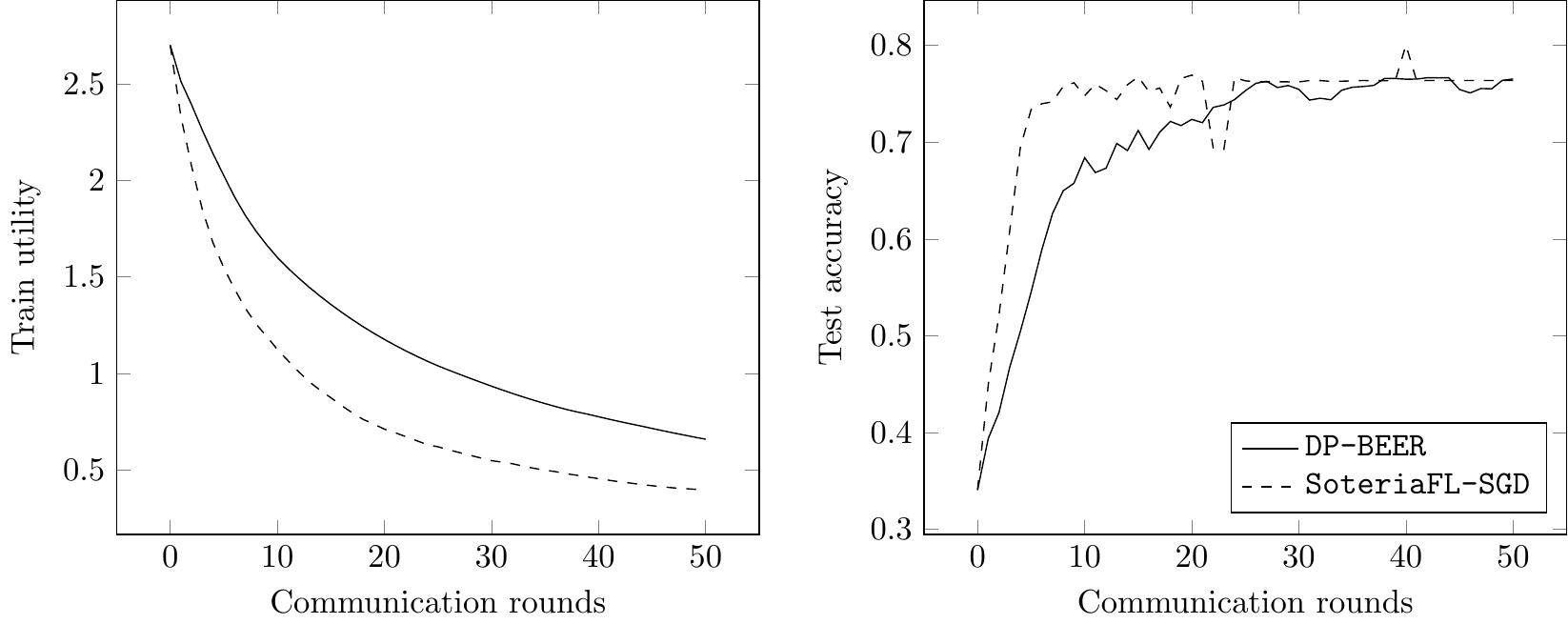} \\
(a) $(10^{-2}, 10^{-3})$-LDP & (b) $(10^{-1}, 10^{-3})$-LDP
\end{tabular}
\caption[Convergence of \myalgdp for logistic regression with nonconvex regularization when guaranteeing $(10^{-2}, 10^{-3})$-LDP]{The train utility and test accuracy vs. communication rounds for logistic regression with nonconvex regularization on the \dataset{a9a} dataset when guaranteeing $(10^{-2}, 10^{-3})$-LDP and $(10^{-1}, 10^{-3})$-LDP, respectively.
Both \myalgdp and \soteriasgd employ $\text{random}_{162}$ compression (cf. \Cref{eg:random_sparsification}). \label{fig:logistic_regression_eps_0.01}}
\end{figure}

\Cref{fig:logistic_regression_eps_0.01} shows the convergence results of \myalgdp and \soteriasgd for logistic regression with nonconvex regularization on the \dataset{a9a} dataset to reach $(10^{-2}, 10^{-3})$-LDP and $(10^{-1}, 10^{-3})$-LDP, respectively.
Under $(10^{-2}, 10^{-3})$-LDP,
which is a stricter privacy setting,
\myalgdp converges faster than \soteriasgd in test accuracy,
while converges slightly slower in train utility.
Under $(10^{-1}, 10^{-3})$-LDP,
\myalgdp performs slightly worse than \soteriasgd. Given that \myalgdp operates under the decentralized topology with much weaker information exchange,
the results highlight \myalgdp's communication efficiency by showing it can achieve similar performance as its server/client counterpart, i.e. \soteriasgd,
especially under strict privacy constraints.

\subsection{One-hidden-layer neural network training}
We evaluate \myalgdp by training a one-hidden layer neural network on the \dataset{MNIST} dataset \citep{lecun1995learning}.  
The network uses $64$ hidden neurons, sigmoid activation functions and cross-entropy loss,
where the loss function over a training pair $\{ \bm f, l \}$ is defined as 
\begin{align*}
\ell(\x; (\bm f, l)) = \mathsf{CrossEntropy}(\mathsf{softmax}(\bm W_2 ~\mathsf{sigmoid}( \bm W_1 \bm f + \bm c_1) + \bm c_2), l).
\end{align*}
Here, the model parameter is defined by $\x = \text{vec}(\bm W_1, \bm c_1, \bm W_2, \bm c_2)$,
where the dimensions of the network parameters $\bm W_1$, $\bm c_1$, $\bm W_2$, $\bm c_2$ are $64 \times 784$, $64 \times 1$, $10 \times 64$, and $10 \times 1$, respectively.

\begin{figure}[ht]
\begin{tabular}{cc}
\includegraphics[width=0.48\textwidth]{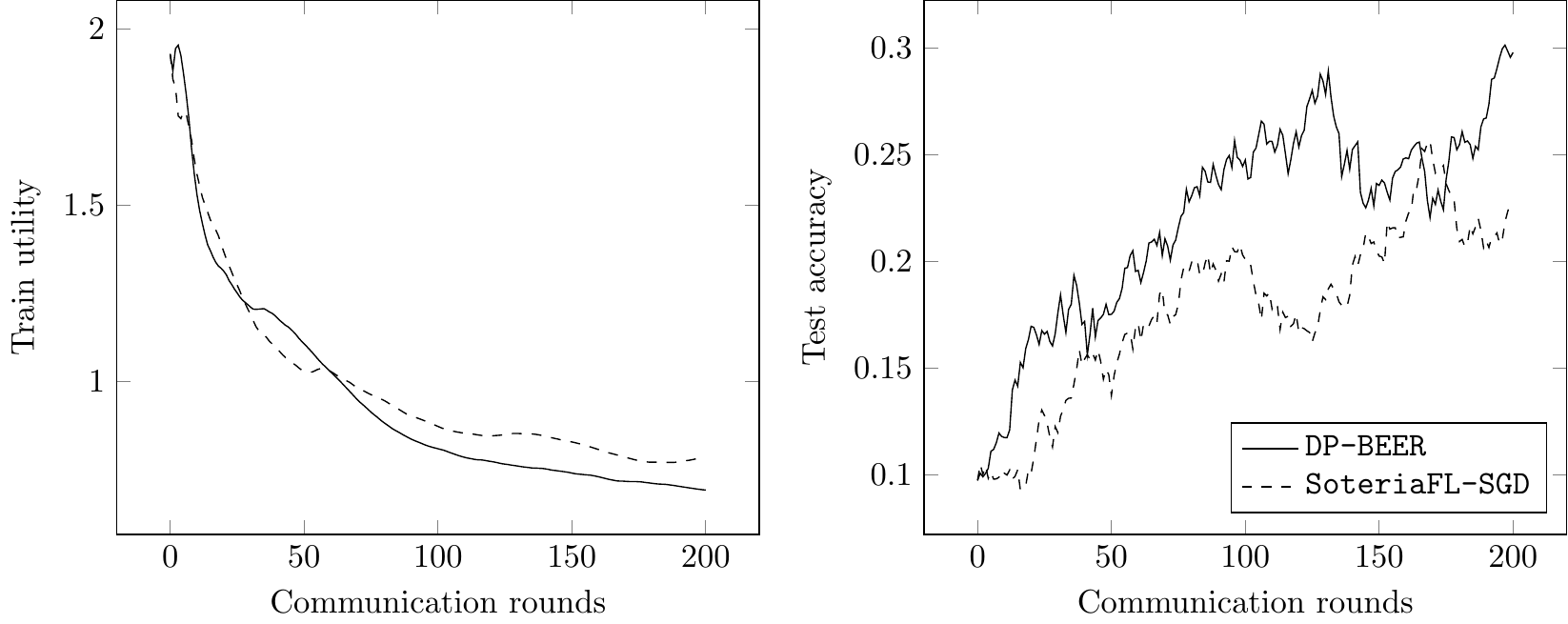} &\includegraphics[width=0.48\textwidth]{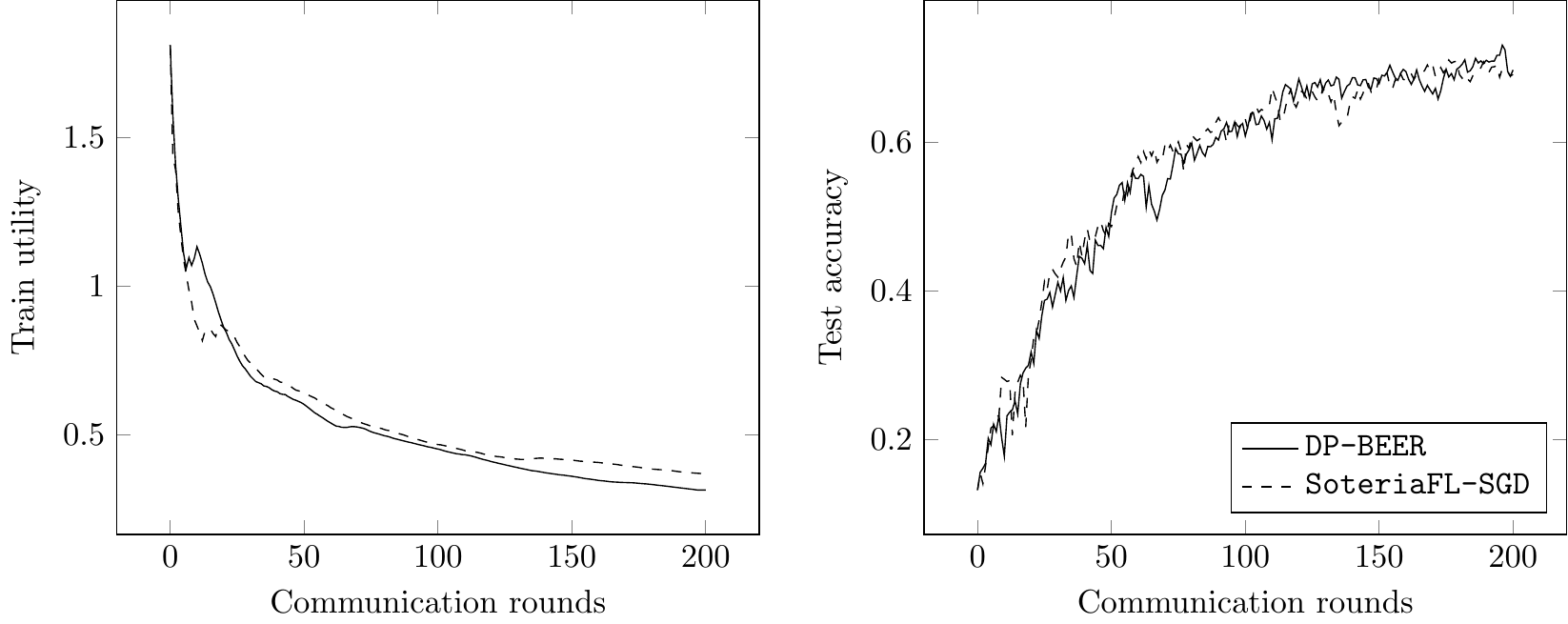} \\
(a) $(10^{-2}, 10^{-3})$-LDP & (b) $(10^{-1}, 10^{-3})$-LDP.
\end{tabular}
\caption[Convergence of \myalgdp for neural network training when guaranteeing $(10^{-2}, 10^{-3})$-LDP]{The train utility and test accuracy vs. communication rounds for training a one-hidden-layer neural network on the \dataset{MNIST} dataset when guaranteeing $(10^{-2}, 10^{-3})$-LDP and $(10^{-1}, 10^{-3})$-LDP, respectively.
Both \myalgdp and \soteriasgd employ $\text{random}_{2583}$ compression (cf. \Cref{eg:random_sparsification}). \label{fig:nn_eps_0.01}}
\end{figure}

\Cref{fig:nn_eps_0.01}
shows the convergence results of \myalgdp and \soteriasgd for training a one-hidden-layer neural network on the \dataset{MNIST} dataset to reach $(10^{-2}, 10^{-3})$-LDP and $(10^{-1}, 10^{-3})$-LDP, respectively.
Under both privacy settings,
\myalgdp converges at a similar rate as \soteriasgd in train utility.
However,
in terms of convergence in test accuracy,
\myalgdp outperforms \soteriasgd under the stricter $(10^{-2}, 10^{-3})$-LDP,
while the two algorithms have similar performance under the other setting.
This experiment again emphasizes \myalgdp's communication efficiency in comparison to the server/client algorithm \soteriasgd.
\section{Conclusions}
\label{sec:conclusions}

In this paper, we propose an algorithmic framework called \myalg, which incorporates practically-relevant gradient clipping and communication compression simultaneously in the design of decentralized nonconvex optimization algorithms. We propose two variants: \myalgdp and \myalgsgd. While they share a similar structure that makes use of gradient tracking, communication compression, and error feedback,
their focuses are on different perspectives motivated by applications in privacy preserving and neural network training, respectively.
\myalgdp adds privacy perturbation to clipped gradients to protect the local differential privacy of each agent,
with explicit utility and communication complexities.
\myalg applies gradient clipping to mini-batch stochastic gradients, which converges in minimum $\ell_2$ gradient norm at similar rate as centralized SGD without clipping under proper choices of hyperparameters.
The development of \myalg offers a simple analysis framework to understand gradient clipping in decentralized nonconvex optimization without bounded gradient assumptions, highlighting the potential of achieving both communication efficiency and privacy preserving in the decentralized framework.
  
\section*{Acknowledgements}
The work of B.~Li and Y.~Chi is supported in part by ONR under the grant N00014-19-1-2404,  by AFRL under FA8750-20-2-0504, and by NSF under the grants CCF-1901199, CCF-2007911 and CNS-2148212. B.~Li is also gratefully supported by the Wei Shen and Xuehong Zhang Presidential Fellowship at Carnegie Mellon University.

\newcommand{\etalchar}[1]{$^{#1}$}

\appendix
\section{Proof of Theorem \ref{theorem:privacy}}
\label{sec:proof_privacy}

This section proves \Cref{theorem:privacy} in the following steps:
1) define privacy loss and moment generating function,
2) define mechanisms and sub-mechanisms, and
3) bound the overall moment generating function and show the choice of perturbation variance satisfies all conditions.

\paragraph{Moment generating function.}
Let $o$ and $\text{aux}$ denote an outcome and an auxiliary input, respectively.
Then, we can define the privacy loss of an outcome $o$ on neighboring dataset $\bZ$ and $\bZ^{(i)}$ as
\begin{align}
    c(o; \cM, \text{aux}, \bZ, \bZ^{(i)})
= \log \frac{\P \big(\cM(\text{aux}, \bZ) = o \big)}{\P \big(\cM(\text{aux}, \bZ^{(i)}) = o \big)},
\notag
\end{align}
and its log moment generating functions as
\begin{align}
    \alpha_i^{\cM}(\lambda; \text{aux}, \bZ, \bZ^{(i)})
    &= \log \E_{o \sim \cM(\text{aux}, \bZ)} \big[ \exp \big( \lambda c(o; \cM, \text{aux}, \bZ, \bZ^{(i)}) \big) \big].
    \notag
\end{align}

Taking maximum over conditions,
the unconditioned log moment generating function is
\begin{align}
    {\hat\alpha}_i^{\cM} (\lambda)
    &= \max_{\text{aux}, \bZ, \bZ^{(i)}} \alpha_i^{\cM} (\lambda; \text{aux}, \bZ, \bZ^{(i)}) .
    \notag
\end{align}

\paragraph{Sub-mechanisms.}
\Cref{def:local_differential_privacy} defines the LDP mechanism,
but it is not enough to model decentralized algorithms.
To model the perturbation operation happens on agent $i$ at time $t$,
we define a sub-mechanism as $\cM_i^{(t)}: \cD \to \cR$, where $i \in [n], t \in [T]$,
which can be understood as the perturbation added on agent $i$ at time $t$.
In addition, we define another mechanism $\cC: \cR \to \cR$ to model the compression operator and $\cC \circ \cM_i^{(t)}$ to represent the full update at an agent,
and use $\cM$ to represent the full algorithm.

\paragraph{Proof of LDP.}
The overall log moment generating function for agent $i$ can be bounded using \cite[Lemma 2]{li2022soteriafl} as
\begin{align}
    \hat\alpha_i^{\cM}(\lambda)
    \leq \sum_{t=1}^T \hat\alpha_i^{\cC \circ \cM_i^{(t)}} (\lambda)
    \leq \sum_{t=1}^T \hat\alpha_i^{\cM_i^{(t)}} (\lambda)
    .
    \notag
\end{align}

Let $q = \frac 1m$ denote the probability each data sample is chosen.
For agent $i$ and $\lambda > 0$,
assume $q \leq \frac{\tau}{16 \sigma_p}$
and $\lambda \leq \frac{\sigma_p^2}{\tau^2} \log \frac{\tau}{q \sigma_p}$.
We can apply \cite[Lemma 3]{abadi2016deep} to bound each $\hat\alpha_i^{\cM_i^{(t)}}(\lambda)$
as
\begin{align}
    \hat\alpha^{\cM_i^{(t)}}(\lambda)
    \leq
    \frac{q^2 \lambda (\lambda + 1) \tau^2}{(1 - q) \sigma_p^2} + O\Big( \frac{q^3 \lambda^3 \tau^3}{\sigma_p^3} \Big)
    = O\Big( \frac{q^2 \lambda^2 \tau^2}{\sigma_p^2}  \Big) .
    \notag
\end{align}

To conclude the proof,
we can verify there exists some $\lambda$ that satisfies the following inequalities
when choosing $\sigma_p = \frac{\tau q \sqrt{T \log(1 / \delta)}}{\epsilon}$ and $q = \frac 1m$, namely,
\begin{align*}
    & \Big( \frac{T q \tau \lambda}{\sigma_p} \Big)^2
    \leq \frac{\lambda \epsilon}{2},
    \\
    & \exp(-\lambda \epsilon / 2)
    \leq \delta, \\
    & \lambda
    \leq \frac{\sigma_p^2}{\tau^2} \log \frac{\tau}{q \sigma_p}.
\end{align*}

\section{Proof of Theorem \ref{theorem:convergence_witht_privacy_with_bounded_gradient}}
\label{proof:convergence_witht_privacy_with_bounded_gradient}

This section proves \Cref{theorem:convergence_witht_privacy_with_bounded_gradient} in the following $4$ subsections: 
\Cref{sub:with_bounded_gradient_function_value_descent} derives the descent inequality,
\Cref{sub:with_bounded_gradient_var_consensus,sub:with_bounded_gradient_grad_consensus} create two linear systems to bound the sum of consensus errors in the descent inequality,
and finally \Cref{sub:with_bounded_gradient_convergence_rate} specifies hyper parameters to obtain convergence rate.

With \Cref{assumption:bounded_stochastic_gradient},
we can skip the compression operator,
in \myalgdp.
To reuse this section's results in later analysis,
we assume \Cref{assumption:bounded_variance} in deriving descent lemma and linear systems,
and lift this assumption when computing convergence rate in \Cref{sub:with_bounded_gradient_convergence_rate} using $\sigma_g \leq 2 \tau$.

Denote
$$\bG^{(t)} = \frac{1}{b} \sum_{\bZ \in \cZ^{(t)}} \nabla \ell (\bX^{(t-1)}; \bZ ).$$

\subsection{Function value descent}
\label{sub:with_bounded_gradient_function_value_descent}
Using Taylor expansion,
and taking expectation conditioned on time $t$,
\begin{align}
    \E_{t} \big[ f(\bbx^{(t + 1)}) - f(\bbx^{(t)}) \big]
    &\leq \E_{t} \langle \nabla f(\bbx^{(t)}), -\eta \bbv^{(t+1)} \rangle
    + \frac L2 \E_{t} \big\| \eta \bbv^{(t+1)} \big\|_2^2 \notag \\
    &= -\eta \langle \nabla f(\bbx^{(t)}),  \E_{t}[\bbv^{(t+1)}] \rangle
    + \frac{\eta^2 L}{2}  \E_t \big\| \bbv^{(t+1)} \big\|_2^2 \notag \\
    &= -\eta \langle \nabla f(\bbx^{(t)}),  \E_{t}[\bbg_\tau^{(t+1)} + \bbe^{(t+1)}] \rangle
    + \frac{\eta^2 L}{2}  \E_t \big\| \bbv^{(t+1)} \big\|_2^2 , \notag
\end{align}
where the last equality is due to $\bbv^{(t)} = \bbg_p^{(t)}$ that can be proved by induction.

Because $\E_t[\be_i^{(t)}] = \bm0_d$ and stochastic gradients are unbiased,
\begin{align}
    &~~~~ \E_{t} \big[ f(\bbx^{(t + 1)}) - f(\bbx^{(t)}) \big] \notag \\
    &= -\eta \langle \nabla f(\bbx^{(t)}),  \nabla F(\bX^{(t)}) \ave \rangle
    + \frac{\eta^2 L}{2} \E_t \big\| \bbv^{(t+1)} \big\|_2^2 \notag \\
    &= \frac{ \eta}{2} \Big( \big\| \nabla f(\bbx^{(t)}) - \nabla F(\bX^{(t)}) \ave \big\|_2^2 - \big\| \nabla f(\bbx^{(t)}) \big\|_2^2 - \big\| \nabla F(\bX^{(t)}) \ave \big\|_2^2 \Big)
    + \frac{\eta^2 L}{2}  \E_t \big\| \bbv^{(t+1)} \big\|_2^2 \notag \\
    &\leq -\frac{\eta}{2} \big\| \nabla f(\bbx^{(t)}) \big\|_2^2
    +\frac{\eta L^2}{2n} \big\| \bX^{(t)} - \bbx^{(t)} \one_n^\top \big\|_\F^2
    + \frac{\eta^2 L}{2}  \E_t \big\| \bbv^{(t+1)} \big\|_2^2
    - \frac{\eta}{2}  \big\| \nabla F(\bX^{(t)}) \ave \big\|_2^2,
    \label{eq:proof_1_descent_1}
\end{align}
where the last inequality is due to \Cref{assumption:smoothness}.

Let $\Delta = \E \big[ f(\bbx^{(0)}) \big] - f^\opt$.
Take full expectation and average \eqref{eq:proof_1_descent_1} over $t=1, \hdots , T$,
the expected utility can be bounded by
\begin{align}
    \frac1T \sum_{t=1}^T \E \big\| \nabla f(\bbx^{(t)}) \big\|_2^2
    &\leq \frac{2 \Delta}{\eta T}
    + \frac{1}{T} \cdot \frac{L^2}{n} \sum_{t=1}^T \E \big\| \bX^{(t)} - \bbx^{(t)}\one_n^\top \big\|_\F^2 \notag \\
    &~~~~ + \frac1T \cdot \eta L \sum_{t=1}^T \E \big\| \bbv^{(t+1)} \big\|_2^2
    - \frac1T \sum_{t=1}^T \E \big\| \nabla F(\bX^{(t)}) \ave \big\|_2^2 .
    \label{eq:proof_1_descent_2}
\end{align}

\subsection{Sum of variable consensus errors}
\label{sub:with_bounded_gradient_var_consensus}

This subsection creates a linear system to bound $\sum_{t=1}^T \E \big\| \bX^{(t)} - \bbx^{(t)}\one_n^\top \big\|_\F^2$ by
$\sum_{t=1}^T \E \big\| \bV^{(t)} - \bbv^{(t)}\one_n^\top \big\|_\F^2$ and $\sum_{t=1}^T \E \big\| \bbv^{(t)} \big\|_2^2$.
To simplify notations, let $\widehat\bW = \bI_n + \gamma(\bW - \bI_n)$,
and denote the mixing rate of $\widehat\bW$ by $\hat\alpha = \big\| \widehat\bW - \mean \big\|_{\mathrm {op}}$.
\Cref{lemma:mixing_rate_of_regularized_mixing_matrix} analyzes the mixing rate of the regularized mxing matrix.
\begin{lemma}[Mixing rate of regularized mixing matrix]
\label{lemma:mixing_rate_of_regularized_mixing_matrix}
Assuming $0 < \gamma \leq 1$.
The mixing rate of $\widehat\bW$ can be bounded as
\begin{align}
\hat\alpha\leq 1 + \gamma (\alpha - 1) .
\end{align}
\end{lemma}

\begin{proof}[Proof of \Cref{lemma:mixing_rate_of_regularized_mixing_matrix}]
Let $\lambda_1 = 1 > \lambda_2 \geq \hdots \geq \lambda_n > -1$ denote the eigenvalues of $\bW$.
Corresponding eigenvalues of $\widehat\bW$ are $1 + \gamma(\lambda_i - 1)$, $i = 1, \hdots, n$.

The mixing rate of $\widehat\bW$ is
\begin{align*}
\hat\alpha
& = \max \big\{\big| 1 + \gamma(\lambda_2 - 1) \big|, \big|1 + \gamma(\lambda_n - 1) \big| \big\} \\
& \leq \max \big\{ |1 - \gamma| + \gamma | \lambda_2 |, |1 - \gamma| + \gamma |\lambda_n | \big\} \\
&= 1 + \gamma (\alpha - 1).
\end{align*}
\end{proof}

\subsubsection{Variable consensus error}
Take expectation conditioned on time $t$,
and use Young's inequality,
the variable consensus error can be bounded as
\begin{align}
    &~~~~ \E_t \big\| \bX^{(t+1)} - \bbx^{(t+1)}\one_n^\top \big\|_\F^2 \notag \\
    &= \E_t \Big\| \Big( \bX^{(t)} + \gamma \bQ_x^{(t+1)} (\bW - \bI_n) - \eta \bV^{(t+1)} \Big) \big(\bI_n - \mean \big) \Big\|_\F^2 \notag \\
    &= \E_t \Big\|
    \big(\bX^{(t)} - \bbx^{(t)} \one_n^\top \big)\big( \widehat\bW - \mean \big)
    + \gamma (\bQ_x^{(t+1)} - \bX^{(t)}) (\bW - \bI_n)
- \eta \bV^{(t+1)} \big(\bI_n - \mean \big) \Big\|_\F^2 \notag \\
    &\leq \frac{2}{1 + \hat\alpha^2} \big\| \big(\bX^{(t)} - \bbx^{(t)} \one_n^\top \big) \big( \widehat\bW - \mean \big) \big\|_\F^2 \notag \\
    &~~~~+ \frac{2}{1 - \hat\alpha^2} \E_t \Big\| \gamma (\bQ_x^{(t+1)} - \bX^{(t)}) (\bW - \bI_n) - \eta \bV^{(t+1)} \big(\bI_n - \mean \big) \Big\|_\F^2 \notag \\
    &\leq \frac{2}{1 + \hat\alpha^2} \big\| \big(\bX^{(t)} - \bbx^{(t)} \one_n^\top \big) \big( \widehat\bW - \mean \big) \big\|_\F^2 \notag \\
    &~~~~+ \frac{4}{1 - \hat\alpha^2} \E_t \Big\| \gamma (\bQ_x^{(t+1)} - \bX^{(t)}) (\bW - \bI_n) \Big\|_\F^2
    + \frac{4}{1 - \hat\alpha^2} \E_t \Big\| \eta \bV^{(t+1)} \big(\bI_n - \mean \big) \Big\|_\F^2 \notag \\
    &\overset{(\text{i})}{\leq} \frac{2\hat\alpha^2}{1 + \hat\alpha^2} \big\| \bX^{(t)} - \bbx^{(t)}\one_n^\top \big\|_\F^2
    + \frac{16 \gamma^2}{1 - \hat\alpha^2} \E_t \big\| \bQ_x^{(t+1)} - \bX^{(t)} \big\|_\F^2
    + \frac{4 \eta^2}{1 - \hat\alpha^2} \E_t \big\| \bV^{(t+1)} \ - \bbv^{(t+1)}\one_n^\top \big\|_\F^2 \notag \\
    &\overset{(\text{ii})}{\leq} \hat\alpha \big\| \bX^{(t)} - \bbx^{(t)}\one_n^\top \big\|_\F^2
    +  \frac{16 (1 - \rho) \gamma^2}{1 - \hat\alpha} \big\| \bQ_x^{(t)} - \bX^{(t)} \big\|_\F^2
    + \frac{4 \eta^2}{1 - \hat\alpha} \E_t \big\| \bV^{(t+1)} \ - \bbv^{(t+1)}\one_n^\top \big\|_\F^2 ,
    \label{eq:proof_1_consensus}
\end{align}
where (i) is obtained by $\| \bW - \bI_n \|_{\mathrm{op}} \leq 2$,
(ii) uses $2 \hat\alpha \leq 1 + \hat\alpha^2$, $1 - \hat\alpha \leq 1 - \hat\alpha^2$ and \Cref{definition:general_compression_operator}.

\subsubsection{Variable quantization error}
Assume $\gamma$ satisfies the following inequality (which will be verified in \Cref{sub:with_bounded_gradient_convergence_rate})
\begin{align}
    \gamma^2 \leq \frac{\rho^2}{96(1 - \rho)} .
    \label{eq:proof_1_gamma_1}
\end{align}

Taking expectation conditioned on time $t$,
the variable quantization error can be decomposed and bounded as
\begin{align}
    &~~~~ \E_t \big\| \bQ_x^{(t+1)} - \bX^{(t+1)} \big\|_\F^2 \notag \\
    &= \E_t \big\| \bQ_x^{(t)} + \cC (\bX^{(t)} - \bQ_x^{(t)}) - \bX^{(t+1)} \big\|_\F^2 \notag \\
    &= \E_t \big\| \cC (\bX^{(t)} - \bQ_x^{(t)}) - (\bX^{(t)} - \bQ_x^{(t)}) - (\bX^{(t+1)} - \bX^{(t)}) \big\|_\F^2 \notag \\
    &\overset{(\text{i})}{\leq} \frac{2}{1 + (1 - \rho)} \E_t \big\| \cC (\bX^{(t)} - \bQ_x^{(t)}) - (\bX^{(t)} - \bQ_x^{(t)}) \big\|_\F^2
    + \frac{2}{1 - (1 - \rho)} \E_t \big\| \bX^{(t+1)} - \bX^{(t)} \big\|_\F^2 \notag \\
    &\overset{(\text{ii})}{\leq} \frac{2(1 - \rho)}{1 + (1 - \rho)} \big\| \bX^{(t)} - \bQ_x^{(t)} \big\|_\F^2
    + \frac{2}{\rho} \E_t \big\| \bX^{(t+1)} - \bX^{(t)} \big\|_\F^2 \notag \\
    &= \frac{2(1 - \rho)}{1 + (1 - \rho)} \big\| \bX^{(t)} - \bQ_x^{(t)} \big\|_\F^2
    \notag\\
    &~~~~ + \frac{2}{\rho} \E_t \big\| \gamma (\bQ_x^{(t+1)} - \bX^{(t)}) (\bW - \bI_n) + \gamma (\bX^{(t)} - \bbx^{(t)}\one_n^\top) (\bW - \bI_n) - \eta \bV^{(t+1)} \big\|_\F^2 \notag \\
    &\overset{(\text{iii})}{\leq} \Big(1 - \frac\rho2\Big) \big\| \bX^{(t)} - \bQ_x^{(t)} \big\|_\F^2
    + \frac{24 \gamma^2}{\rho} \E_t \big\| \bQ_x^{(t+1)} - \bX^{(t)} \big\|_\F^2
    + \frac{24 \gamma^2}{\rho} \big\| \bX^{(t)} - \bbx^{(t)}\one_n^\top \big\|_\F^2
    + \frac{6 \eta^2}{\rho} \E_t \big\| \bV^{(t+1)} \big\|_\F^2 \notag \\
    &\leq \Big( 1 - \frac\rho2 + \frac{24 (1 - \rho) \gamma^2}{\rho} \Big) \big\| \bQ_x^{(t)} - \bX^{(t)} \big\|_\F^2
    + \frac{24 \gamma^2}{\rho} \big\| \bX^{(t)} - \bbx^{(t)}\one_n^\top \big\|_\F^2
    + \frac{6 \eta^2}{\rho} \E_t \big\| \bV^{(t+1)} \big\|_\F^2
    \notag\\
    &\overset{(\text{iv})}{\leq} \Big( 1 - \frac\rho4 \Big) \big\| \bQ_x^{(t)} - \bX^{(t)} \big\|_\F^2
    + \frac{24 \gamma^2}{\rho} \big\| \bX^{(t)} - \bbx^{(t)}\one_n^\top \big\|_\F^2
    + \frac{6 \eta^2}{\rho} \E_t \big\| \bV^{(t+1)} \big\|_\F^2,
    \label{eq:proof_1_grad_consensus}
\end{align}
where (i) is obtained by applying Young's inequality,
(ii) uses \Cref{definition:general_compression_operator},
(iii) uses the fact $\big\| \bW - \bI_n \big\|_{\text{op}} \leq 2$,
and (iv) uses \eqref{eq:proof_1_gamma_1}.

\subsubsection{Linear system}
\label{sub:proof_1_variable_linear_system}

Let $\be_1^{(t)} =
\begin{bmatrix}
\big\| \bX^{(t)} - \bbx^{(t)}\one_n^\top \big\|_\F^2 \\
\big\| \bQ_x^{(t)} - \bX^{(t)} \big\|_\F^2
\end{bmatrix}$,
we can take full expectation and rewrite \eqref{eq:proof_1_consensus} and \eqref{eq:proof_1_grad_consensus} in matrix form as
\begin{align}
    \E [\be_1^{(t+1)}]
    &\leq
    \begin{bmatrix}
        \hat\alpha & \frac{16 (1 - \rho) \gamma^2}{1 - \hat\alpha} \\
        \frac{24 \gamma^2}{\rho} & 1 - \frac\rho4
    \end{bmatrix}
    \E [\be_1^{(t)}]
    +
    \begin{bmatrix}
        \frac{4 \eta^2}{1 - \hat\alpha} \E \big\| \bV^{(t+1)} \ - \bbv^{(t+1)}\one_n^\top \big\|_\F^2 \\
        \frac{6 \eta^2}{\rho} \E \big\| \bV^{(t+1)} \big\|_\F^2
    \end{bmatrix}
    \notag \\
    &:= \bG_1 \E [\be_1^{(t)}] + \bb_1^{(t)} .
    \label{eq:proof_1_consensus_error_linear_system_1}
\end{align}

We can compute $( \bI_n - \bG_1 )^{-1}$ and verify all its entries are positive:
\begin{align}
    (\bI_n - \bG_1)^{-1}
    &=
    \frac{1}{(1- \hat\alpha) \cdot \frac\rho4 - \frac{16 (1 - \rho) \gamma^2}{1 - \hat\alpha} \cdot \frac{24 \gamma^2}{\rho}}
    \begin{bmatrix}
        \frac\rho4 & \frac{16 (1 - \rho) \gamma^2}{1 - \hat\alpha} \\
        \frac{24 \gamma^2}{\rho} & 1 - \hat\alpha
    \end{bmatrix} \notag \\
    &\leq
    \frac{1}{\frac18 (1- \hat\alpha) \rho}
    \begin{bmatrix}
        \frac\rho4 & \frac{16 (1 - \rho) \gamma^2}{1 - \hat\alpha} \\
        \frac{24 \gamma^2}{\rho} & 1 - \hat\alpha
    \end{bmatrix},
    \label{eq:proof_1_variable_linear_system_inverse}
\end{align}
where we assume the following inequality to to prove \eqref{eq:proof_1_variable_linear_system_inverse}, which will be validated in \Cref{sub:with_bounded_gradient_convergence_rate}:
\begin{align}
    &(1 - \hat\alpha) \cdot \frac\rho4 - \frac{16 (1 - \rho) \gamma^2}{1 - \hat\alpha} \cdot \frac{24 \gamma^2}{\rho}
    \geq \frac{1}{8} (1 - \hat\alpha)\rho .
    \label{eq:proof_1_gamma_2}
\end{align}

Sum expected error vectors $\E[\be_1^{(t)}]$ over $t = 1, \hdots, T$,
\begin{align}
    \sum_{t=1}^T \E[\be_1^{(t)}]
    &\leq \sum_{t=1}^T ( \bG_1 \E[\be_1^{(t-1)}] + \bb_1^{(t-1)}) \notag \\
    &\leq \bG_1 \sum_{t=1}^T \E[\be_1^{(t)}]
    + \bG_1 \E[\be_1^{(0)}]
    + \sum_{t=1}^T \bb_1^{(t-1)} \notag .
\end{align}

Reorganize terms, multiply $(\bI_n - \bG_1)^{-1}$ on both sides and use $\be_1^{(0)} = \bm0_2$,
the sum of error vectors can be bounded as
\begin{align}
    \sum_{t=1}^T \E[\be_1^{(t)}]
    &\leq (\bI_n - \bG_1)^{-1} \sum_{t=0}^{T-1} \bb_1^{(t)}.
\end{align}

The sum of consensus error can be computed as
\begin{align}
    &~~~~ \sum_{t=1}^T \E \big\| \bX^{(t)} - \bbx^{(t)}\one_n^\top \big\|_\F^2 \notag \\
    &\leq
    \begin{bmatrix}
        1 & 0
    \end{bmatrix}
    (\bI_n - \bG)^{-1} \sum_{t=0}^{T-1} \bb_1^{(t)} \notag \\
    &= \frac{1}{\frac{1}{8}(1 - \hat\alpha) \rho}
    \begin{bmatrix}
        \frac\rho4 & \frac{16 (1 - \rho) \gamma^2}{1 - \hat\alpha}
    \end{bmatrix}
    \begin{bmatrix}
        \frac{4 \eta^2}{1 - \hat\alpha} \sum_{t=1}^T \E \big\| \bV^{(t)} \ - \bbv^{(t)}\one_n^\top \big\|_\F^2 \\
        \frac{6 \eta^2}{\rho} \sum_{t=1}^T \E \big\| \bV^{(t)} \big\|_\F^2
    \end{bmatrix} \notag \\
    &= \frac{\eta^2}{\frac{1}{8}(1 - \hat\alpha) \rho} \Big( \frac{\rho}{1 - \hat\alpha} \sum_{t=1}^T  \E \big\| \bV^{(t)} - \bbv^{(t)}\one_n^\top \big\|_\F^2
    + \frac{16 (1 - \rho) \gamma^2}{1 - \hat\alpha} \cdot \frac{6}{\rho} \sum_{t=1}^T \E \big\| \bV^{(t)} \big\|_\F^2 \Big) \notag \\
    &\overset{(\text{i})}{=} \frac{8 \eta^2}{(1 - \hat\alpha) \rho} \Big( \frac{\rho}{1 - \hat\alpha} + \frac{96 (1 - \rho) \gamma^2}{(1 - \hat\alpha) \rho} \Big) \sum_{t=1}^T  \E \big\| \bV^{(t)} \ - \bbv^{(t)}\one_n^\top \big\|_\F^2
    + \frac{768 (1 - \rho) \gamma^2 \eta^2}{(1 - \hat\alpha)^2\rho^2} \sum_{t=1}^T n \E \big\| \bbv^{(t)} \big\|_\F^2 \notag \\
    &\overset{(\text{ii})}{\leq} \frac{16 \eta^2}{(1 - \hat\alpha)^2} \sum_{t=1}^T  \E \big\| \bV^{(t)} \ - \bbv^{(t)}\one_n^\top \big\|_\F^2
    + \frac{768 (1 - \rho) \gamma^2 \eta^2}{(1 - \hat\alpha)^2\rho^2} \sum_{t=1}^T n \E \big\| \bbv^{(t)} \big\|_\F^2,
    \label{eq:proof_1_consensus_error_4}
\end{align}
where we use the equality
$\big\| \bV^{(t)} \big\|_\F^2 = \big\| \bV^{(t)} - \bbv^{(t)}\one_n^\top \big\|_\F^2 + n \| \bbv^{(t)} \|_2^2$ for (i) and use \eqref{eq:proof_1_gamma_1} for (ii).

\subsection{Sum of gradient consensus errors}
\label{sub:with_bounded_gradient_grad_consensus}

This section creates a linear system to bound the sum of gradient consensus error $\sum_{t=1}^T \E \big\| \bV^{(t)} - \bbv^{(t)} \one_n^\top \big\|_\F^2$ by $\sum_{t=1}^T \E \big\| \bbv^{(t)} \big\|_2^2$ and constant terms.

\subsubsection{Gradient consensus error}

Take expectation conditioned on time $t$ and reorganize terms,
the gradient consensus error can be expanded as
\begin{align}
    &~~~~ \E_t \big\| \bV^{(t+1)} - \bbv^{(t+1)}\one_n^\top \big\|_\F^2 \notag \\
    &= \E_t \Big\| \Big( \bV^{(t)} + \gamma \bQ_v^{(t+1)}(\bW - \bI_n) + \bG_p^{(t+1)} - \bG_p^{(t)} \Big) \big(\bI_n - \mean \big) \Big\|_\F^2 \notag \\
    &= \E_t \Big\|
    \big(\bV^{(t)} - \bbv^{(t)} \one_n^\top\big) \big( \widehat\bW - \mean \big)
    + \gamma \big(\bQ_v^{(t+1)} - \bV^{(t)}\big) (\bW - \bI_n)
    + \big( \bG_p^{(t+1)} - \bG_p^{(t)} \big) \big(\bI_n - \mean \big)
    \Big\|_\F^2. \notag 
\end{align}

Then, take full expectation,
use the update formula and Young's inequality similarly to \eqref{eq:proof_1_consensus},
\begin{align}
    &~~~~ \E \big\| \bV^{(t+1)} - \bbv^{(t+1)}\one_n^\top \big\|_\F^2 \notag \\
    &\leq \frac{2\hat\alpha^2}{1 + \hat\alpha^2} \big\| \bV^{(t)} - \bbv^{(t)} \one_n^\top \big\|_\F^2
    \notag\\
    &~ + \frac{2}{1 - \hat\alpha^2} \E \Big\| \gamma \big(\bQ_v^{(t+1)} - \bV^{(t)}\big) (\bW - \bI_n)
    + \big( \bG_p^{(t+1)} - \bG_p^{(t)} \big) \big(\bI_n - \mean \big) \Big\|_\F^2 \notag \\
    &\leq \hat\alpha \big\| \bV^{(t)} - \bbv^{(t)} \one_n^\top \big\|_\F^2
    + \frac{4}{1 - \hat\alpha} \E \big\| \gamma \big(\bQ_v^{(t+1)} - \bV^{(t)}\big) (\bW - \bI_n) \big\|_\F^2
    \notag\\
    &~~~~ + \frac{4}{1 - \hat\alpha} \E \big\| \big(\bG_p^{(t+1)} - \bG_p^{(t)}\big)\big(\bI_n - \mean \big) \big\|_\F^2
    \notag \\
    &\overset{(\text{i})}{\leq} \hat\alpha \big\| \bV^{(t)} - \bbv^{(t)} \one_n^\top \big\|_\F^2
    + \frac{16 \gamma^2}{1 - \hat\alpha} \E \big\| \bQ_v^{(t+1)} - \bV^{(t)} \big\|_\F^2
    + \frac{4}{1 - \hat\alpha} \E \big\| \bG_p^{(t+1)} - \bG_p^{(t)} \big\|_\F^2
    \notag \\
    &\overset{(\text{ii})}{\leq} \hat\alpha \big\| \bV^{(t)} - \bbv^{(t)} \one_n^\top \big\|_\F^2
    + \frac{16 (1 - \rho) \gamma^2}{1 - \hat\alpha} \E \big\| \bQ_v^{(t)} - \bV^{(t)} \big\|_\F^2
    + \frac{16 n (\tau^2 + \sigma_p^2 d)}{1 - \hat\alpha} ,
    \label{eq:proof_1_gradient_consensus_1}
\end{align}
where (i) is proved using the facts $\big\| \bW - \bI_n \big\|_{\mathsf{op}} \leq 2$ and $\big\| \bI_n - \mean \big\|_{\mathsf{op}} \leq 1$,
(ii) is due to \Cref{definition:general_compression_operator} and
\begin{align}
    \E \big\| \bG_p^{(t+1)} - \bG_p^{(t)} \big\|_\F^2
    &\leq 2 \E \big\| \bG_p^{(t+1)} \big\|_\F^2 + 2 \E \big\| \bG_p^{(t)} \big\|_\F^2
    \notag\\
    &= 2 \big( \E \big\| \bG_\tau^{(t+1)} \big\|_\F^2 + n \sigma_p^2 d \big) + 2 \big( \E \big\| \bG_\tau^{(t)} \big\|_\F^2 + n \sigma_p^2 d \big)
    \notag\\
    &\leq 4n (\tau^2 + \sigma_p^2 d) .
    \label{eq:proof_1_g_p}
\end{align}

\subsubsection{Gradient quantization error}
\begin{align}
    \E_t \big\| \bQ_v^{(t+1)} - \bV^{(t+1)} \big\|_\F^2
    &= \E_t \big\| (\bQ_v^{(t+1)} - \bV^{(t)}) - (\bV^{(t+1)} - \bV^{(t)}) \big\|_\F^2 \notag \\
    &\leq \frac{2}{1 + (1 - \rho)} \E_t \big\| \bQ_v^{(t+1)} - \bV^{(t)} \big\|_\F^2
    + \frac{2}{1 - (1 - \rho)} \E_t \big\| \bV^{(t+1)} - \bV^{(t)} \big\|_\F^2 \notag \\
    &\leq \frac{2 (1 - \rho)}{2 - \rho} \big\| \bQ_v^{(t)} - \bV^{(t)} \big\|_\F^2
    + \frac{2}{\rho} \E_t \big\| \gamma \bQ_v^{(t+1)} (\bW - \bI_n) + \bG_p^{(t+1)} - \bG_p^{(t)} \big\|_\F^2 \notag \\
    &\leq \frac{2 (1 - \rho)}{2 - \rho} \big\| \bQ_v^{(t)} - \bV^{(t)} \big\|_\F^2
    + \frac{6 \gamma^2}{\rho} \E_t \big\| (\bQ_v^{(t+1)} - \bV^{(t)}) (\bW - \bI_n) \big\|_\F^2 \notag \\
    &~~~~ + \frac{6 \gamma^2}{\rho} \big\| \bV^{(t)} (\bW - \bI_n) \big\|_\F^2
    + \frac{6}{\rho} \E_t \big\| \bG_p^{(t+1)} - \bG_p^{(t)} \big\|_\F^2 \notag \\
    &\overset{\text{(i)}}{\leq} \Big(1 - \frac\rho2 + \frac{24 \gamma^2 (1 - \rho)}{\rho} \Big) \big\| \bQ_v^{(t)} - \bV^{(t)} \big\|_\F^2
    + \frac{24 \gamma^2}{\rho} \big\| \bV^{(t)} - \bbv^{(t)}\one_n^\top \big\|_\F^2
    + \frac{24 n (\tau^2 + \sigma_p^2 d)}{\rho} \notag\\
    &\overset{\text{(ii)}}{\leq} \Big(1 - \frac\rho4 \Big) \big\| \bQ_v^{(t)} - \bV^{(t)} \big\|_\F^2
    + \frac{24 \gamma^2}{\rho} \big\| \bV^{(t)} - \bbv^{(t)}\one_n^\top \big\|_\F^2
    + \frac{24 n (\tau^2 + \sigma_p^2 d)}{\rho} ,
    \label{eq:proof_1_gradient_consensus_2}
\end{align}
where we use \eqref{eq:proof_1_g_p} and the fact $\frac{2 (1 - \rho)}{2 - \rho} = 1 - \frac{\rho}{2 - \rho} \geq 1 - \frac\rho2$ when $\rho \geq 0$ to reach (i) and use \eqref{eq:proof_1_gamma_1} to reach (ii).

\subsubsection{Linear system}

Let $\be_2^{(t)} =
\begin{bmatrix}
\big\| \bV^{(t)} - \bbv^{(t)}\one_n^\top \big\|_\F^2 \\
\big\| \bQ_v^{(t)} - \bV^{(t)} \big\|_\F^2
\end{bmatrix}$.
We can write \eqref{eq:proof_1_gradient_consensus_1} and \eqref{eq:proof_1_gradient_consensus_2} in matrix form as
\begin{align}
    \E [\be_2^{(t+1)}]
    &\leq
    \begin{bmatrix}
        \hat\alpha & \frac{16 (1 - \rho) \gamma^2}{1 - \hat\alpha} \\
        \frac{24 \gamma^2}{\rho} & 1 - \frac\rho4
    \end{bmatrix}
    \E [\be_2^{(t)}]
    +
    \begin{bmatrix}
        \frac{16 n (\tau^2 + \sigma_p^2 d)}{1 - \hat\alpha} \\
        \frac{24 n (\tau^2 + \sigma_p^2 d)}{\rho} \\
    \end{bmatrix} \notag \\
    &:= \bG_2 \E [\be_2^{(t)}] + \bb_2^{(t)} \notag .
\end{align}

Because $\bG_2 = \bG_1$,
we can use the same argument as in \Cref{sub:proof_1_variable_linear_system},
and use \eqref{eq:proof_1_variable_linear_system_inverse} to prove
\begin{align}
    \sum_{t=1}^T \E \big\| \bV^{(t)} - \bbv^{(t)} \one_n^\top \big\|_\F^2
    &\leq
    \begin{bmatrix}
        1 & 0
    \end{bmatrix}
    (\bI_n - \bG_2)^{-1} \Big( \E [\be_2^{(0)}] + \sum_{t=0}^{T-1} \bb_2^{(t)} \Big) \notag \\
    &\leq \frac{1}{\frac{1}{8}(1 - \hat\alpha) \rho}
    \begin{bmatrix}
        \frac\rho4 & \frac{16 (1 - \rho) \gamma^2}{1 - \hat\alpha}
    \end{bmatrix}
    \begin{bmatrix}
        \frac{16 T n (\tau^2 + \sigma_p^2 d)}{1 - \hat\alpha} \\
        \frac{24 T n (\tau^2 + \sigma_p^2 d)}{\rho} \\
    \end{bmatrix} \notag \\
    &= \frac{T n (\tau^2 + \sigma_p^2 d)}{\frac{1}{8}(1 - \hat\alpha) \rho} \cdot \Big( \frac\rho4 \cdot \frac{16}{1 - \hat\alpha} +  \frac{16 (1 - \rho) \gamma^2}{1 - \hat\alpha} \cdot \frac{24}{\rho} \Big) \notag \\
    &\leq \frac{T n (\tau^2 + \sigma_p^2 d)}{\frac{1}{8}(1 - \hat\alpha) \rho} \cdot \Big( \frac{4 \rho}{1 - \hat\alpha} + \frac{4 \rho}{1 - \hat\alpha} \Big) \notag \\
    &= \frac{64}{(1 - \hat\alpha)^2} T n (\tau^2 + \sigma_p^2 d),
    \label{eq:proof_1_consensus_error_5}
\end{align}
where we use \eqref{eq:proof_1_gamma_1} to prove the last inequality.

With \eqref{eq:proof_1_consensus_error_5},
we can bound \eqref{eq:proof_1_consensus_error_4} by
\begin{align}
    &~~~~ \sum_{t=1}^T \E \big\| \bX^{(t)} - \bbx^{(t)}\one_n^\top \big\|_\F^2 \notag \\
    &\leq \frac{16 \eta^2}{(1 - \hat\alpha)^2} \sum_{t=1}^T  \E \big\| \bV^{(t)} \ - \bbv^{(t)}\one_n^\top \big\|_\F^2
    + \frac{768 (1 - \rho) \gamma^2 \eta^2}{(1 - \hat\alpha)^2\rho^2} \sum_{t=1}^T n \E \big\| \bbv^{(t)} \big\|_\F^2 \notag \\
    &\leq \frac{16 \eta^2}{(1 - \hat\alpha)^2} \cdot \frac{64}{(1 - \hat\alpha)^2} T n (\tau^2 + \sigma_p^2 d)
    + \frac{768 (1 - \rho) \gamma^2 \eta^2}{(1 - \hat\alpha)^2\rho^2} \sum_{t=1}^T n \E \big\| \bbv^{(t)} \big\|_\F^2 \notag \\
    &\leq \frac{1024 \eta^2}{(1 - \hat\alpha)^4} T n (\tau^2 + \sigma_p^2 d)
    + \frac{8 \eta^2}{(1 - \hat\alpha)^2} \sum_{t=1}^T n \E \big\| \bbv^{(t)} \big\|_\F^2 ,
    \label{eq:proof_1_consensus_error_6}
\end{align}
where we use \eqref{eq:proof_1_gamma_1} again to prove the last inequality.

\subsection{Convergence rate}
\label{sub:with_bounded_gradient_convergence_rate}

Note bounded gradient assumption can imply \Cref{assumption:bounded_variance} for some $\sigma_g \leq 2\tau$,
we can bound the expected norm of average gradient estimate as
\begin{align}
    \E \big\| \bbv^{(t)} \big\|_2^2
    &= \E \big\| \bbg_p^{(t)} \big\|_2^2 \notag\\
    &= \E \big\| \bbg_\tau^{(t)} \big\|_2^2 + \frac{\sigma_p^2 d}{n} \notag\\
    &\leq \E \big\| \nabla F(\bX^{(t)})\ave \big\|_2^2 + \frac{\sigma_g^2}{b} + \frac{\sigma_p^2 d}{n} \notag\\
    &\leq \E \big\| \nabla F(\bX^{(t)})\ave \big\|_2^2 + \frac{4\tau^2}{b} + \frac{\sigma_p^2 d}{n} .
    \label{eq:proof_1_consensus_error_7}
\end{align}

We assume
\begin{align}
    \eta L \leq \frac18(1 - \hat\alpha)^{\frac43} .
    \label{eq:proof_1_eta}
\end{align}

Using \eqref{eq:proof_1_consensus_error_6} \eqref{eq:proof_1_consensus_error_7},
expected utility \eqref{eq:proof_1_descent_2} can be bounded by
\begin{align}
    \frac1T \sum_{t=1}^T \E \| \nabla f(\bbx^{(t)}) \|_2^2
    &\leq \frac{2 \Delta}{\eta T}
    + \frac{1}{T} \cdot \frac{L^2}{n} \sum_{t=1}^T \E \big\| \bX^{(t)} - \bbx^{(t)}\one_n^\top \big\|_\F^2 \notag \\
    &~~~~ + \frac1T \sum_{t=1}^T \eta L \E \| \bbv^{(t)} \|_2^2
    - \frac1T \sum_{t=1}^T \E \big\| \nabla F(\bX^{(t)}) \ave \big\|_2^2 \notag \\
    &\leq \frac{2 \Delta}{\eta T}
    + \frac1T \cdot \frac{L^2}{n} \Big( \frac{1024 \eta^2}{(1 - \hat\alpha)^4} T n (\tau^2 + \sigma_p^2 d)
    + \frac{8 \eta^2}{(1 - \hat\alpha)^2} \sum_{t=1}^T n \E \big\| \bbv^{(t)} \big\|_2^2 \Big)
     \notag \\
    &~~~~ + \frac1T \sum_{t=1}^T \eta L \E \| \bbv^{(t)} \|_2^2
    - \frac1T \sum_{t=1}^T \E \big\| \nabla F(\bX^{(t)}) \ave \big\|_2^2 \notag \\
    &\overset{(\text{i})}{\leq} \frac{2 \Delta}{\eta T}
    + \frac{1024 \eta^2 L^2}{(1 - \hat\alpha)^4} (\tau^2 + \sigma_p^2 d)
    + \frac{2 \eta L}{(1 - \hat\alpha)^{\frac43}T} \sum_{t=1}^T \E \| \bbv^{(t)} \|_2^2
    - \frac1T \sum_{t=1}^T \E \big\| \nabla F(\bX^{(t)}) \ave \big\|_2^2 \notag \\
    &\leq \frac{2 \Delta}{\eta T}
    + \frac{1024 \eta^2 L^2}{(1 - \hat\alpha)^4} (\tau^2 + \sigma_p^2 d)
    + \frac{2 \eta L}{(1 - \hat\alpha)^{\frac43}} \Big(\frac{4\tau^2}{b} + \frac{\sigma_p^2 d}{n} \Big) \notag \\
    &\overset{(\text{ii})}{=} \frac{2 \Delta}{\eta T}
    + \frac{1024 \eta^2 L^2 \tau^2}{(1 - \hat\alpha)^4} (1 + T \phi_m^2)
    + \frac{8 \eta L \tau^2}{(1 - \hat\alpha)^{\frac43}} (1 + T \phi_m^2) \notag \\
    &\overset{(\text{iii})}{=}
    \frac{2 \Delta}{\eta T}
    + \frac{2048 \eta^2 L^2 \tau^2}{(1 - \hat\alpha)^4} 
    + \frac{16 \eta L \tau^2}{(1 - \hat\alpha)^{\frac43}} 
    \label{eq:privacy_biased_bounded_gradient_1}
\end{align}
where we use \eqref{eq:proof_1_eta} for (i),
substitute $b = 1$
and $\sigma_p^2 d = \Big(\frac{\tau \sqrt{T \log(1/\delta)}} {m \epsilon} \Big)^2 = T \tau^2 \phi_m^2$ for (ii),
and substitute $T = \phi_m^{-2}$ for (iii).

We set the step size as
\begin{align}
    \eta
    &= \frac{\gamma^{\frac43}(1 - \alpha)^{\frac43}}{32} \cdot \frac{\phi_m }{L}, \notag
\end{align}
\eqref{eq:privacy_biased_bounded_gradient_1} can be further bounded as
\begin{align}
    \frac1T \sum_{t=1}^T \E \| \nabla f(\bbx^{(t)}) \|_2^2
    &\leq \frac{64 L \Delta \phi_m}{\gamma^{\frac43}(1 - \alpha)^{\frac43}}
    + \frac{2 \tau^2 \phi_m^2 }{(1 - \hat\alpha)^{\frac43}}
    + \frac{\tau^2 \phi_m}{2}  \notag \\
    &\leq \frac{64 L \Delta \phi_m}{\gamma^{\frac43}(1 - \alpha)^{\frac43}}
    + \frac{3 \tau^2 \phi_m }{(1 - \hat\alpha)^{\frac43}} \notag \\
    &\leq \frac{67 \phi_m}{\gamma^{\frac43}(1 - \alpha)^{\frac43}} \max \big\{\tau^2, L\Delta \big\} ,
    \notag
\end{align}
where we use \Cref{lemma:mixing_rate_of_regularized_mixing_matrix} to reach the last inequality.

Lastly, set the hyper parameter $\gamma$ as
\begin{align}
    \gamma
    &= \frac{1}{100}(1-\alpha) \rho \notag.
\end{align}

We can now verify conditions
\eqref{eq:proof_1_gamma_1},
\eqref{eq:proof_1_gamma_2}
and the condition on $\eta$ are all met to conclude the proof:
\begin{align}
    & \gamma^2
    \leq \frac{\rho^2}{10000}
    &&\Rightarrow 
    \tag{\ref{eq:proof_1_gamma_1}}
    \\
    & \gamma^4
    = \gamma^2 \cdot \frac{(1-\alpha)^2\rho^2}{10000}
    \leq \frac{(1-\hat\alpha)^2\rho^2}{10000}
    &&\Rightarrow
    \tag{\ref{eq:proof_1_gamma_2}}
    \\
    & \eta L
    \leq \frac{(1 - \hat\alpha)^{\frac43}}{32}
    &&\Rightarrow
    \tag{\ref{eq:proof_1_eta}} 
\end{align}
\section{Proof of Theorem \ref{theorem:convergence_witht_privacy_without_bounded_gradient}}
\label{proof:convergence_with_privacy_without_bounded_gradient}

This section proves \Cref{theorem:convergence_witht_privacy_without_bounded_gradient} in $2$ subsections.
\Cref{sub:function_value_descent} derives the descent inequality using results from \Cref{sub:with_bounded_gradient_var_consensus,sub:with_bounded_gradient_grad_consensus}.
\Cref{sub:convergence_rate} first assumes all expected gradient norm $\E \big\| \nabla f(\bbx^{(t)}) \big\|_2$ are greater than a threshold $\nu$ (i.e. $\E \big\| \nabla f(\bbx^{(t)}) \big\|_2 \geq \nu$ for all $t = 1, \hdots, T$),
then specifies parameters and proves the average of expected gradient norm is smaller than that threshold $\frac1T \sum_{t=1}^T \E \big\| \nabla f(\bbx^{(t)}) \big\|_2 \leq \nu$,
which contradicts the assumption hence proves the algorithm reaches $\E \big\| \nabla f(\bbx^{(t)}) \big\|_2 \leq \nu$ within $T$ steps.

\subsection{Function value descent}
\label{sub:function_value_descent}

Let $\delta_i^{(t)} = \frac{\tau}{\tau + \| \bg_i^{(t)} \|_2}$
and $\delta^{(t)} = \frac{\tau}{\tau + \| \nabla f(\bbx^{(t)}) \|_2}$.
Similar to \Cref{sub:with_bounded_gradient_function_value_descent},
use Taylor expansion and take expectation conditioned on $t$,
we can expand the function value descent as
\begin{align}
    &~~~~ \E_{t} \big[ f(\bbx^{(t + 1)}) - f(\bbx^{(t)}) \big] \notag \\
    &\leq \E_{t} \langle \nabla f(\bbx^{(t)}), -\eta \bbv^{(t+1)} \rangle
    + \frac L2 \E_{t} \big\| \eta \bbv^{(t+1)} \big\|_2^2 \notag \\
    &= -\eta \E_t \big\langle \nabla f(\bbx^{(t)}), \bbg_p^{(t+1)} \big\rangle
    + \frac{\eta^2 L}{2}  \E_t \big\| \bbv^{(t+1)} \big\|_2^2 \notag \\
    &= -\eta \E_t \big\langle \nabla f(\bbx^{(t)}), \bbg_\tau^{(t+1)} \big\rangle
    + \frac{\eta^2 L}{2}  \E_t \big\| \bbv^{(t+1)} \big\|_2^2 \notag \\
    &= -\eta \E_t \big\langle \nabla f(\bbx^{(t)}), \clip_\tau(\nabla f(\bbx^{(t)})) \big\rangle
    +\eta \E_t \big\langle \nabla f(\bbx^{(t)}), \clip_\tau(\nabla f(\bbx^{(t)})) - \bbg_\tau^{(t+1)} \big\rangle
    + \frac{\eta^2 L}{2}  \E_t \big\| \bbv^{(t+1)} \big\|_2^2 \notag \\
    &= -\eta \delta^{(t)} \big\| \nabla f(\bbx^{(t)}) \big\|_2^2
    +\eta \E_t \big\langle \nabla f(\bbx^{(t)}), \clip_\tau(\nabla f(\bbx^{(t)})) - \bbg_\tau^{(t+1)} \big\rangle
    + \frac{\eta^2 L}{2}  \E_t \big\| \bbv^{(t+1)} \big\|_2^2 \notag \\
    &\leq -\eta \delta^{(t)} \big\| \nabla f(\bbx^{(t)}) \big\|_2^2
    + \eta \big\| \nabla f(\bbx^{(t)}) \big\|_2 \E_t \big\| \clip_\tau(\nabla f(\bbx^{(t)})) - \bbg_\tau^{(t+1)} \big\|_2
    + \frac{\eta^2 L}{2}  \E_t \big\| \bbv^{(t+1)} \big\|_2^2 .
    \label{eq:proof_2_descent_1}
\end{align}
The $\E_t \big\| \clip_\tau(\nabla f(\bbx^{(t)})) - \bbg_\tau^{(t+1)} \big\|_2$ term in \eqref{eq:proof_2_descent_1}
is the error introduced by gradient clipping,
which can be analyzed by splitting it to $4$ terms as following
\begin{align}
    &~~~~ \E_t \big\| \clip_\tau(\nabla f(\bbx^{(t)})) - \bbg_\tau^{(t+1)} \big\|_2 \notag \\
    &= \E_t \Big\| \frac1n \sum_{i=1}^n \frac{\tau}{\tau + \| \bg_i^{(t)} \|_2} \bg_i^{(t)} - \frac{\tau}{\tau + \| \nabla f(\bbx^{(t)}) \|_2} \nabla f(\bbx^{(t)}) \Big\|_2 \notag \\
    &= \E_t \Bigg\| \frac1n \sum_{i=1}^n \Big(
    \frac{\tau}{\tau + \| \bg_i^{(t)} \|_2} \bg_i^{(t)}
    - \frac{\tau}{\tau + \| \nabla f_i(\x_i^{(t)}) \|_2} \bg_i^{(t)} \Big) \notag\\
    &~~~~~~~+
    \frac1n \sum_{i=1}^n \Big(
    \frac{\tau}{\tau + \| \nabla f_i(\x_i^{(t)}) \|_2} \bg_i^{(t)}
    - \frac{\tau}{\tau + \| \nabla f_i(\x_i^{(t)}) \|_2} \nabla f_i(\x_i^{(t)}) \Big)  \notag\\
    &~~~~~~~+
    \frac1n \sum_{i=1}^n \Big(
    \frac{\tau}{\tau + \| \nabla f_i(\x_i^{(t)}) \|_2} \nabla f_i(\x_i^{(t)})
    - \frac{\tau}{\tau + \| \nabla f(\bbx^{(t)}) \|_2} \nabla f_i(\x_i^{(t)}) \Big)  \notag \\
    &~~~~~~~+
    \frac1n \sum_{i=1}^n \Big(
    \frac{\tau}{\tau + \| \nabla f(\bbx^{(t)}) \|_2} \nabla f_i(\x_i^{(t)})
    - \frac{\tau}{\tau + \| \nabla f(\bbx^{(t)}) \|_2} \nabla f(\bbx^{(t)}) \Big) 
    \Bigg\|_2 \notag \\
    &\leq \frac1n \sum_{i=1}^n \E_t \Bigg\| 
    \Big( \frac{\tau}{\tau + \| \bg_i^{(t)} \|_2}
    - \frac{\tau}{\tau + \| \nabla f_i(\x_i^{(t)}) \|_2} \Big) \bg_i^{(t)}
    \Bigg\|_2
    \label{eq:proof_2_clipping_error_1}
    \\
    &\qquad\qquad + \frac1n \sum_{i=1}^n \E_t \Bigg\| 
    \frac{\tau}{\tau + \| \nabla f_i(\x_i^{(t)}) \|_2} \big( \bg_i^{(t)} - \nabla f_i(\x_i^{(t)}) \big)
    \Bigg\|_2
    \label{eq:proof_2_clipping_error_2}
    \\
    &\qquad\qquad + \frac1n \sum_{i=1}^n \Bigg\| 
    \Big( \frac{\tau}{\tau + \| \nabla f_i(\x_i^{(t)}) \|_2}
    - \frac{\tau}{\tau + \| \nabla f(\bbx^{(t)}) \|_2} \Big) \nabla f_i(\x_i^{(t)})
    \Bigg\|_2
    \label{eq:proof_2_clipping_error_3}
    \\
    &\qquad\qquad+ \Bigg\| 
    \frac{\tau}{\tau + \| \nabla f(\bbx^{(t)}) \|_2} \Big( \frac1n \sum_{i=1}^n \nabla f_i(\x_i^{(t)}) - \nabla f(\bbx^{(t)})
    \Big) \Bigg\|_2 .
    \label{eq:proof_2_clipping_error_4}
\end{align}

Next, we bound each term separately
using triangle inequality,
\Cref{assumption:bounded_variance,assumption:bounded_similarity}.

Bound the first term \eqref{eq:proof_2_clipping_error_1} as
\begin{align}
    &~~~~ \frac1n \sum_{i=1}^n \E_t \Bigg\|
    \Big( \frac{\tau}{\tau + \big\| \bg_i^{(t)} \big\|_2}
    - \frac{\tau}{\tau + \big\| \nabla f_i(\x_i^{(t)}) \big\|_2} \Big) \bg_i^{(t)}
    \Bigg\|_2 \notag \\
    &= \frac1n \sum_{i=1}^n \E_t \Bigg\|
    \frac{\tau( \big\| \bg_i^{(t)} \big\|_2 - \big\| \nabla f_i(\x_i^{(t)}) \big\|_2)}{(\tau + \big\| \bg_i^{(t)} \big\|_2)(\tau + \big\| \nabla f_i(\x_i^{(t)}) \big\|_2)} \bg_i^{(t)}
    \Bigg\|_2 \notag \\
    &= \frac1n \sum_{i=1}^n \E_t \Bigg( \Big|
    \big\| \bg_i^{(t)} \big\|_2 - \big\| \nabla f_i(\x_i^{(t)}) \big\|_2 \Big|
    \cdot
    \frac{\tau}{\tau + \big\| \nabla f_i(\x_i^{(t)}) \big\|_2}
    \cdot
    \frac{\big\| \bg_i^{(t)} \big\|_2}{\tau + \big\| \bg_i^{(t)} \big\|_2}
    \Bigg)
    \notag \\
    &\leq \frac1n \sum_{i=1}^n \E_t \Big|
    \big\| \bg_i^{(t)} \big\|_2 - \big\| \nabla f_i(\x_i^{(t)}) \big\|_2
    \Big| \notag \\
    &\leq \frac1n \sum_{i=1}^n \sqrt{\E_t \big(
    \big\| \bg_i^{(t)} \big\|_2 - \big\| \nabla f_i(\x_i^{(t)}) \big\|_2
    \big)^2} \notag \\
    &= \frac1n \sum_{i=1}^n \sqrt{\E_t \big(
    \big\| \bg_i^{(t)} \big\|_2^2 + \big\| \nabla f_i(\x_i^{(t)}) \big\|_2^2
    -2 \big\| \bg_i^{(t)} \big\|_2 \big\| \nabla f_i(\x_i^{(t)}) \big\|_2
    \big)} \notag \\
    &\leq \frac1n \sum_{i=1}^n \sqrt{\E_t \big(
    \big\| \bg_i^{(t)} \big\|_2^2 + \big\| \nabla f_i(\x_i^{(t)}) \big\|_2^2
    - 2 \langle \bg_i^{(t)}, \nabla f_i(\x_i^{(t)}) \rangle
    \big)} \notag \\
    &= \frac1n \sum_{i=1}^n \sqrt{\E_t 
    \big\| \bg_i^{(t)} - \nabla f_i(\x_i^{(t)}) \big\|_2^2
    } \notag \\
    &\leq
    \frac{\sigma_g}{\sqrt{b}} .
    \label{eq:proof_2_clipping_error_bound_1}
\end{align}

Bound the second term \eqref{eq:proof_2_clipping_error_2} as
\begin{align}
    \frac1n \sum_{i=1}^n \E_t \Bigg\|
    \frac{\tau}{\tau + \big\| \nabla f_i(\x_i^{(t)}) \big\|_2} \big(\bg_i^{(t)} - \nabla f_i(\x_i^{(t)}) \big)
    \Bigg\|_2
    \leq
    \frac1n \sum_{i=1}^n \frac{\tau \sigma_g / \sqrt{b}}{\tau + \big\| \nabla f_i(\x_i^{(t)}) \big\|_2}
    \leq
    \frac{\sigma_g}{\sqrt{b}} .
    \label{eq:proof_2_clipping_error_bound_2}
\end{align}

Bound the third term \eqref{eq:proof_2_clipping_error_3} as
\begin{align}
    &~~~~ \frac1n \sum_{i=1}^n \Bigg\|
    \Big( \frac{\tau}{\tau + \big\| \nabla f_i(\x_i^{(t)}) \big\|_2}
    - \frac{\tau}{\tau + \big\| \nabla f(\bbx^{(t)}) \big\|_2} \Big) \nabla f_i(\x_i^{(t)})
    \Bigg\|_2 \notag\\
    &= \frac1n \sum_{i=1}^n \Bigg\|
    \frac{\tau \big(\big\| \nabla f_i(\x_i^{(t)}) \big\|_2 - \big\| \nabla f(\bbx^{(t)}) \big\|_2 \big)}{\big( \tau + \big\| \nabla f_i(\x_i^{(t)}) \big\|_2 \big) \big( \tau + \big\| \nabla f(\bbx^{(t)}) \big\|_2 \big)}
    \nabla f_i(\x_i^{(t)})
    \Bigg\|_2 \notag\\
    &\leq
    \frac1n \sum_{i=1}^n \frac{\tau \Big| \big\| \nabla f_i(\x_i^{(t)}) \big\|_2 - \big\| \nabla f(\bbx^{(t)}) \big\|_2 \Big| }{\tau + \big\| \nabla f(\bbx^{(t)}) \big\|_2}
    \notag\\
    &\leq
    \frac1n \sum_{i=1}^n \delta^{(t)} \Big| \big\| \nabla f_i(\x_i^{(t)}) \big\|_2 - \big\| \nabla f_i(\bbx^{(t)}) \big\|_2 \Big|
    + \frac1n \sum_{i=1}^n \delta^{(t)} \Big| \big\| \nabla f_i(\bbx^{(t)}) \big\|_2 - \big\| \nabla f(\bbx^{(t)}) \big\|_2 \Big|
    \notag\\
    &\leq
    \frac1n \sum_{i=1}^n \delta^{(t)} L \big\| \x_i^{(t)} - \bbx^{(t)} \big\|_2
    + \frac1n \sum_{i=1}^n \delta^{(t)} \cdot \frac{1}{12} \| \nabla f(\bbx^{(t)}) \|_2
    \notag\\
    &\leq
    \frac{\delta^{(t)} L}{\sqrt n } \big\| \bX^{(t)} - \bbx^{(t)}\one_n^\top \big\|_\F
    + \frac{1}{12} \| \nabla f(\bbx^{(t)}) \|_2 ,
    \label{eq:proof_2_clipping_error_bound_3}
\end{align}
where we use $\delta^{(t)} \leq 1$ to reach the last inequality.

Bound \eqref{eq:proof_2_clipping_error_4} as
\begin{align}
    &~~~~ \Bigg\| 
    \frac{\tau}{\tau + \| \nabla f(\bbx^{(t)}) \|_2} \Big( \frac1n \sum_{i=1}^n \nabla f_i(\x_i^{(t)}) - \nabla f(\bbx^{(t)})
    \Big) \Bigg\|_2 \notag \\
    &= \frac{\tau}{\tau + \big\| \nabla f(\bbx^{(t)}) \big\|_2}
    \big\| \nabla F(\bX^{(t)})\ave - \nabla f(\bbx^{(t)})
    \big\|_2 \notag \\
    &\leq \frac{\delta^{(t)} L}{\sqrt{n}} \big\| \bX^{(t)} - \bbx^{(t)}\one_n^\top \big\|_\F .
    \label{eq:proof_2_clipping_error_bound_4}
\end{align}

Using \eqref{eq:proof_2_clipping_error_bound_1}, \eqref{eq:proof_2_clipping_error_bound_2}, \eqref{eq:proof_2_clipping_error_bound_3} and \eqref{eq:proof_2_clipping_error_bound_4},
the function value descent inequality \eqref{eq:proof_2_descent_1} becomes
\begin{align}
    \E_{t} \big[ f(\bbx^{(t + 1)}) - f(\bbx^{(t)}) \big]
    &\leq -\eta \delta^{(t)} \big\| \nabla f(\bbx^{(t)}) \big\|_2^2
    + \frac{\eta^2 L}{2}  \E_t \big\| \bbv^{(t+1)} \big\|_2^2 \notag\\
    &~~~~+ \eta \big\| \nabla f(\bbx^{(t)}) \big\|_2 \E_t \big\| \clip_\tau \big(\nabla f(\bbx^{(t)}) \big) - \bbg_\tau^{(t+1)} \big\|_2
    \notag\\
    &\leq -\eta \delta^{(t)} \big\| \nabla f(\bbx^{(t)}) \big\|_2^2
    + \frac{\eta^2 L}{2}  \E_t \big\| \bbv^{(t+1)} \big\|_2^2 \notag\\
    &~~~~+ \eta \big\| \nabla f(\bbx^{(t)}) \big\|_2 \Big(\frac{2 \sigma_g}{\sqrt{b}} + \frac{1}{12} \delta^{(t)} \| \nabla f(\bbx^{(t)}) \|_2 + \frac{2 \delta^{(t)} L}{\sqrt{n}} \big\| \bX^{(t)} - \bbx^{(t)}\one_n^\top \big\|_\F \Big)
    \notag\\
    &= -\frac{11}{12} \eta \delta^{(t)} \big\| \nabla f(\bbx^{(t)}) \big\|_2^2
    + \frac{\eta^2 L}{2} \E_t \big\| \bbv^{(t+1)} \big\|_2^2 \notag\\
    &~~~~+ \eta \big\| \nabla f(\bbx^{(t)}) \big\|_2 \Big(\frac{2 \sigma_g}{\sqrt{b}} + \frac{2 \delta^{(t)} L}{\sqrt{n}} \big\| \bX^{(t)} - \bbx^{(t)}\one_n^\top \big\|_\F \Big)
    \notag\\
    &\leq -\frac{5}{12} \eta \delta^{(t)} \big\| \nabla f(\bbx^{(t)}) \big\|_2^2
    + \frac{\eta^2 L}{2}  \E_t \big\| \bbv^{(t+1)} \big\|_2^2 \notag\\
    &~~~~+ \frac{2 \eta \sigma_g}{\sqrt{b}} \big\| \nabla f(\bbx^{(t)}) \big\|_2
    + \frac{2 \delta^{(t)} \eta L^2}{n} \big\| \bX^{(t)} - \bbx^{(t)}\one_n^\top \big\|_\F^2 ,
    \label{eq:proof_2_descent_2}
\end{align}
where the last inequality is due to
\begin{align}
    &~~~~ \eta \big\| \nabla f(\bbx^{(t)}) \big\|_2 \cdot \frac{2 \delta^{(t)} L}{\sqrt{n}} \big\| \bX^{(t)} - \bbx^{(t)}\one_n^\top \big\|_\F
    \notag\\
    &\leq \eta\delta^{(t)} \cdot 2 \cdot \sqrt{\frac12 \big\| \nabla f(\bbx^{(t)}) \big\|_2^2} \cdot \sqrt{\frac{2L^2}{n} \big\| \bX^{(t)} - \bbx^{(t)}\one_n^\top \big\|_\F^2}
    \notag\\
    &\leq \frac12 \eta\delta^{(t)} \big\| \nabla f(\bbx^{(t)}) \big\|_2^2 + \frac{2\delta^{(t)} \eta L^2}{n} \big\| \bX^{(t)} - \bbx^{(t)}\one_n^\top \big\|_\F^2 .
    \notag
\end{align}

\subsection{Convergence rate}
\label{sub:convergence_rate}

Different from \eqref{eq:proof_1_consensus_error_7},
with the use of gradient clipping operator,
we can only bound the expected norm of average gradient estimate as
\begin{align}
    \E \| \bbv^{(t)} \|_2^2
    &= \E \| \bbg_p^{(t)} \|_2^2 \notag\\
    &= \E \| \bbg_\tau^{(t)} + \bbe^{(t)} \|_2^2 \notag\\
    &= \E \| \bbg_\tau^{(t)} \|_2^2 + \E \| \bbe^{(t)} \|_2^2 \notag\\
    &\leq \tau^2 + \frac{\sigma_p^2 d}{n} .
    \label{eq:proof_2_vbar}
\end{align}

Let $\Delta = \E \big[ f(\bbx^{(0)}) \big] - f^\star$.
The techniques used is similar to that used in \Cref{proof:convergence_witht_privacy_with_bounded_gradient},
so that we can reuse results from \Cref{sub:with_bounded_gradient_var_consensus,sub:with_bounded_gradient_grad_consensus},
namely \eqref{eq:proof_1_consensus_error_6},
in the following proof.
Take full expectation and use \eqref{eq:proof_2_vbar},
sum \eqref{eq:proof_2_descent_2} over $t=1, \hdots , T$,
\begin{align}
    -\Delta
    &\leq 
    -\frac{5\eta}{12} \sum_{t=1}^T \E \Big( \delta^{(t)} \big\| \nabla f(\bbx^{(t)}) \big\|_2^2 \Big)
    + \frac{2 \eta \sigma}{\sqrt{b}} \sum_{t=1}^T \E \big\| \nabla f(\bbx^{(t)}) \big\|_2 \notag\\
    &\qquad\qquad + \frac{2 \eta L^2}{n} \sum_{t=1}^T \E \big\| \bX^{(t)} - \bbx^{(t)}\one_n^\top \big\|_\F^2
    + \frac{\eta^2 L}{2} \sum_{t=1}^T \E \big\| \bbv^{(t+1)} \big\|_2^2
    \notag\\
    &\leq
    -\frac{5\eta}{12} \sum_{t=1}^T \E \Big( \delta^{(t)} \big\| \nabla f(\bbx^{(t)}) \big\|_2^2 \Big)
    + \frac{2 \eta \sigma}{\sqrt{b}} \sum_{t=1}^T \E \big\| \nabla f(\bbx^{(t)}) \big\|_2 \notag\\
    &\qquad\qquad+ \frac{2 \eta L^2}{n} \Big( \frac{8 \eta^2}{(1 - \hat\alpha)^2} \sum_{t=1}^T n \E \big\| \bbv^{(t)} \big\|_2^2
    + \frac{1024 \eta^2}{(1 - \hat\alpha)^4} T n (\tau^2 + \sigma_p^2 d) \Big)
    + \frac{\eta^2 L}{2} \sum_{t=1}^T \E \big\| \bbv^{(t+1)} \big\|_2^2
    \notag\\
    &=
    -\frac{5\eta}{12} \sum_{t=1}^T \E \Big( \delta^{(t)} \big\| \nabla f(\bbx^{(t)}) \big\|_2^2 \Big)
    + \frac{2 \eta \sigma}{\sqrt{b}} \sum_{t=1}^T \E \big\| \nabla f(\bbx^{(t)}) \big\|_2 \notag\\
    &\qquad\qquad
    + \frac{16 \eta^3 L^2}{(1 - \hat\alpha)^2} \Big(\tau^2 + \frac{\sigma_p^2 d}{n} \Big)
    + \frac{2048 \eta^3L^2}{(1 - \hat\alpha)^4} T (\tau^2 + \sigma_p^2 d)
    + \frac{\eta^2 L T}{2} \Big(\tau^2 + \frac{\sigma_p^2 d}{n} \Big).
    \label{eq:proof_2_descent_3}
\end{align}

To be able to analyze the expected clipped gradient norm $\E \big( \delta^{(t)} \big\| \nabla f(\bbx^{(t)}) \big\|_2^2 \big)$,
we need to use convexity and monotonicity from \Cref{lemma:util_function}.
\begin{lemma}
    \label{lemma:util_function}
    Let $g(x) = \frac{x}{c + x}$ and $h(x) = x g(x) = \frac{x^2}{c + x}$.
    When $x \geq 0$, $g(x)$ and $h(x)$ increase monotonically,
    while $g(x)$ is concave and $h(x)$ is convex.
\end{lemma}

\begin{proof}[Proof of \Cref{lemma:util_function}]
    It is sufficient to prove \Cref{lemma:util_function} by evaluating the first-order and second-order derivatives of $g(x)$ and $h(x)$.

Because
$g'(x) = \frac{(c + x) - x}{(c + x)^2} = \frac{c}{(c + x)^2} > 0$
and $h'(x) = g(x) + x g'(x) \geq 0$,
$g(x)$ and $h(x)$ increase monotonically.

$g(x)$ is concave because $g''(x)
-\frac{2c (c + x)}{(c + x)^4}
= -\frac{2c}{(c + x)^3} < 0$.

$h(x)$ is convex because
$h''(x)
= 2g'(x) + xg''(x)
= \frac{2c}{(c + x)^2} -\frac{2cx}{(c + x)^3}
= \frac{2c^2}{(c + x)^3} > 0$.
\end{proof}

Next,
we substitute $\tau = \nu$ (cf. \Cref{theorem:convergence_witht_privacy_with_bounded_gradient}),
and assume the following inequality
\begin{align}
    \E \| \nabla f(\bbx^{(t)}) \|_2 \geq \nu.
    \label{eq:clipping_contradiction_assumption}
\end{align}

By \Cref{lemma:util_function},
the expectation of clipped gradients can be bounded as
\begin{align}
    \E \Big( \delta^{(t)} \big\| \nabla f(\bbx^{(t)}) \big\|_2^2 \Big)
    &= \E \Big( \frac{\tau \big\| \nabla f(\bbx^{(t)}) \big\|_2^2}{\tau + \big\| \nabla f(\bbx^{(t)}) \big\|_2} \Big) \notag\\
    &\geq \frac{\tau \big( \E \big\| \nabla f(\bbx^{(t)}) \big\|_2 \big)^2}{\tau + \E \big\| \nabla f(\bbx^{(t)}) \big\|_2} \notag\\
    &\geq \frac{\tau \nu}{\tau + \nu} \E \big\| \nabla f(\bbx^{(t)}) \big\|_2
    \notag\\
    &= \frac{\nu}{2} \E \big\| \nabla f(\bbx^{(t)}) \big\|_2 .
    \label{eq:proof_2_expectation_1}
\end{align}

Using
\eqref{eq:proof_2_vbar},
\eqref{eq:proof_2_expectation_1}
\Cref{assumption:bounded_similarity,assumption:bounded_variance},
and set $b = 1$,
we can further bound the RHS of \eqref{eq:proof_2_descent_3} as
\begin{align}
    -\Delta
    &\leq
    -\frac{5 \eta \nu}{24} \sum_{t=1}^T \E \big\| \nabla f(\bbx^{(t)}) \big\|_2
    + \frac{2 \eta \sigma_g}{\sqrt{b}} \sum_{t=1}^T \E \big\| \nabla f(\bbx^{(t)}) \big\|_2 
    \notag\\
    &\qquad\qquad
    + \frac{16 \eta^3 L^2}{(1 - \hat\alpha)^2} \Big(\tau^2 + \frac{\sigma_p^2 d}{n} \Big)
    + \frac{2048 \eta^3L^2}{(1 - \hat\alpha)^4} T (\tau^2 + \sigma_p^2 d)
    + \frac{\eta^2 L T}{2} \Big(\tau^2 + \frac{\sigma_p^2 d}{n} \Big)
    \label{eq:privacy_proof_clipping_descent_for_later_1}
    \\
    &\overset{\text{(i)}}{\leq} - \frac{\eta \nu}{8} \sum_{t=1}^T \E \big\| \nabla f(\bbx^{(t)}) \big\|_2
    + \frac{2048 T \eta^3L^2}{(1 - \hat\alpha)^4} (\tau^2 + \sigma_p^2 d)
    + \frac{3 T \eta^2 L}{1 - \hat\alpha} \Big(\tau^2 + \frac{\sigma_p^2 d}{n} \Big)
    \notag\\
    &= - \frac{\eta \nu}{8} \sum_{t=1}^T \E \big\| \nabla f(\bbx^{(t)}) \big\|_2
    + \frac{2048 T \eta^3L^2 \tau^2}{(1 - \hat\alpha)^4} (1 + T \phi_m^2)
    + \frac{3 T \eta^2 L \tau^2}{1 - \hat\alpha} (1 + T \phi_m^2) ,
    \label{eq:privacy_proof_clipping_descent_1}
\end{align}
where we use the condition $\nu \geq 24 \sigma_g$ to prove (i) and substitute $\sigma_p^2 = T \tau^2 \phi_m^2 / d$ to prove \eqref{eq:privacy_proof_clipping_descent_1}.

Reorganize terms,
\eqref{eq:privacy_proof_clipping_descent_1} can be further bounded as
\begin{align}
    \frac1T \sum_{t=1}^T \E \big\| \nabla f(\bbx^{(t)}) \big\|_2
    &\leq 
    \frac{8 \Delta}{\eta \nu T}
    + \frac{16384 \eta^2L^2 \nu}{(1 - \hat\alpha)^4} (1 + T \phi_m^2)
    + \frac{24 \eta L \nu} {1 - \hat\alpha} (1 + T \phi_m^2)
    \notag\\
    &\overset{\text{(i)}}{\leq}
    \frac{8 \Delta \phi_m}{\eta \nu}
    + \frac{32768 \eta^2L^2 \nu}{(1 - \hat\alpha)^4}
    + \frac{48 \eta L \nu} {1 - \hat\alpha}
    \notag\\
    &\overset{\text{(ii)}}{\leq}
    \frac{8 \Delta \phi_m}{\eta \nu}
    + \frac{4096 \eta L \nu}{(1 - \hat\alpha)^{\frac83}}
    + \frac{48 \eta L \nu} {1 - \hat\alpha} ,
    \notag\\
    &\leq
    \frac{8 \Delta \phi_m}{\eta \nu}
    + \frac{4144 \eta L \nu}{(1 - \hat\alpha)^{\frac83}} ,
    \label{eq:privacy_proof_clipping_descent_2}
\end{align}
where we substitute $T = \phi_m^{-2}$ to prove (i),
and use \eqref{eq:proof_1_eta} for (ii).

Set $\eta = \frac{\gamma^{\frac83} (1 - \alpha)^{\frac83}}{8288 L}$
and $\gamma = \frac{1}{100}(1 - \alpha)\rho$,
\eqref{eq:privacy_proof_clipping_descent_2} can be further bounded as
\begin{align}
    \frac1T \sum_{t=1}^T \E \big\| \nabla f(\bbx^{(t)}) \big\|_2
    &< \frac{8288}{\gamma^{\frac83} (1 - \alpha)^{\frac83}} \cdot \frac{192 L \Delta \phi_m}{\nu}
    + \frac{\gamma^{\frac83} (1 - \alpha)^{\frac83}}{8288} \cdot \frac{4144 \nu}{(1 - \hat\alpha)^{\frac83}}
    \notag\\
    &= \frac{8288}{\gamma^{\frac83} (1 - \alpha)^{\frac83}} \cdot
    \frac{192 L \Delta \phi_m}{\nu}
    + \frac\nu2
    \notag\\
    &< \frac{1784 \sqrt{L \Delta \phi_m}}{\gamma^{\frac43} (1 - \alpha)^{\frac43}} .
    \label{eq:proof_2_utility}
\end{align}

Choosing $\nu = \frac{1784 \sqrt{L \Delta \phi_m}}{\gamma^{\frac43} (1 - \alpha)^{\frac43}}$,
\eqref{eq:proof_2_utility} simplifies to
$\frac1T \sum_{t=1}^T \E \big\| \nabla f(\bbx^{(t)}) \big\|_2 < \nu$,
which further implies that there exists some $ t \in [T]$ such that $\E \big\| \nabla f(\bbx^{(t)}) \big\|_2 < \nu$.
However,
this contradicts the assumption \eqref{eq:clipping_contradiction_assumption},
which leads to the convergence results in the theorem.

Lastly, we can verify conditions
\eqref{eq:proof_1_gamma_1},
\eqref{eq:proof_1_gamma_2} and \eqref{eq:proof_1_eta}
are all met,
which concludes the proof.
\section{Proof of Theorem \ref{theorem:convergence_without_privacy_without_bounded_gradient}}
\label{proof:convergence_without_privacy_without_bounded_gradient}

This section proves \Cref{theorem:convergence_without_privacy_without_bounded_gradient} based on results from \Cref{proof:convergence_with_privacy_without_bounded_gradient}.
We first assume all expected gradient norm $\E \big\| \nabla f(\bbx^{(t)}) \big\|_2$ are greater than a threshold $\nu$ (i.e. $\E \big\| \nabla f(\bbx^{(t)}) \big\|_2 \geq \nu$ for all $t = 1, \hdots, T$).
Set $b = \big( \frac{24 \sigma_g}{\nu}\big)^2$ and $\sigma_p = 0$ in \eqref{eq:privacy_proof_clipping_descent_for_later_1},
we can reach the following descent inequality
\begin{align}
    -\Delta
    &\leq - \frac{\eta \nu}{8} \sum_{t=1}^T \E \big\| \nabla f(\bbx^{(t)}) \big\|_2
    + \frac{2048 T \eta^3L^2 \tau^2}{(1 - \hat\alpha)^4}
    + \frac{3 T \eta^2 L \tau^2}{1 - \hat\alpha} .
\end{align}

Reorganize terms,
\begin{align}
    \frac1T \sum_{t=1}^T \E \big\| \nabla f(\bbx^{(t)}) \big\|_2
    &\leq 
    \frac{8 \Delta}{\eta \nu T}
    + \frac{16384 \eta^2L^2 \nu}{(1 - \hat\alpha)^4} 
    + \frac{24 \eta L \nu} {1 - \hat\alpha} 
    \notag\\
    &<
    \frac{8 \Delta}{\eta \nu T}
    + \frac{2048 \eta L \nu}{(1 - \hat\alpha)^{\frac83}}
    + \frac{24 \eta L \nu} {1 - \hat\alpha} ,
    \notag\\
    &\leq
    \frac{8 \Delta}{\eta \nu T}
    + \frac{2072 \eta L \nu}{(1 - \hat\alpha)^{\frac83}} ,
    \notag
\end{align}
where we use \eqref{eq:proof_1_eta} for the second inequality.

Set $\eta = \frac1L \cdot \sqrt{\frac{8}{2072}}$ and $\nu = \sqrt{\frac{L \Delta}{(1 - \hat\alpha)^{\frac83} T}}$,
the average $\ell_2$ norm of gradients can be bounded as
\begin{align}
    \frac1T \sum_{t=1}^T \E \big\| \nabla f(\bbx^{(t)}) \big\|_2
    &<
    \sqrt{\frac{L \Delta}{(1 - \hat\alpha)^{\frac83} T}}
    = \nu ,
    \notag
\end{align}
which implies that there exists some $ t \in [T]$, such that $\E \big\| \nabla f(\bbx^{(t)}) \big\|_2 < \nu$ which contradicts the assumption,
and proves that \myalgsgd reaches $\E \big\| \nabla f(\bbx^{(t)}) \big\|_2 \leq \nu$ within $T$ iterations

\end{document}